\documentclass[twoside]{article}

\usepackage[utf8]{inputenc}
\usepackage{amsfonts}
\usepackage{amsmath,amssymb}
\usepackage{nccmath}
\usepackage{mathrsfs}
\usepackage{graphicx}
\usepackage{subfigure}
\usepackage{tabularx}
\usepackage[all]{xy}
\usepackage{color}
\usepackage{multirow}
\usepackage{booktabs}
\usepackage{enumitem}
\usepackage{hyperref}

\newtheorem{proposition}{Proposition}
\newtheorem{lemma}[proposition]{Lemma}
\newtheorem{theorem}[proposition]{Theorem}

\newtheorem{remark}[proposition]{Remark}

\newtheorem{example}{Example}
\newtheorem{assumption}{Assumption}

\newenvironment{proof}[1][Proof\ ]{\medskip\noindent{\bf #1}\ }{%
\hfill $\Box$\par\quad\par}

\usepackage[round]{natbib}

\usepackage{times}
\usepackage{enumitem}
\usepackage{multirow}
\usepackage{setspace} 

\def\mcl#1{\mathcal{#1}}
\def\bracket#1{\left\langle #1\right\rangle}

\def\nn{\nonumber}
\def\opn{\operatorname}
\def\mr{\mathrm}

\def\red#1{\textcolor{black}{#1}}
\def\r#1{\mathbb{R}^{#1}}

\newtheorem{myremark}{Remark}

\newenvironment{mythm}[1][]{\medskip\par\noindent{\bfseries #1}\ \,\,\em}{\medskip\par}

\DeclareSymbolFont{EulerExtension}{U}{euex}{m}{n}
\DeclareMathSymbol{\euintop}{\mathop} {EulerExtension}{"52}
\DeclareMathSymbol{\euointop}{\mathop} {EulerExtension}{"48}

\begin{document}

\title{Koopman-based generalization bound:\\ 
New aspect for full-rank weights}

\author{Yuka Hashimoto$^{1,2}$\quad Sho Sonoda$^{2}$\quad Isao Ishikawa$^{3,2}$\quad Atsushi Nitanda$^{4}$\quad Taiji Suzuki$^{5,2}$\medskip\\
{\normalsize 1. NTT Corporation}\\
{\normalsize 2. Center for Advanced Intelligence Project, RIKEN}\\
{\normalsize 3. Center for Data Science, Ehime University}\\
{\normalsize 4. A*STAR Centre for Frontier AI and Research}\\
{\normalsize 5. Graduate School of Information Science and Technology, The University of Tokyo}}

\date{}

\maketitle

\begin{abstract}
We propose a new bound for generalization of neural networks using Koopman operators. Whereas most of existing works focus on low-rank weight matrices, we focus on full-rank weight matrices. Our bound is tighter than existing norm-based bounds when the condition numbers of weight matrices are small. Especially, it is completely independent of the width of the network if the weight matrices are orthogonal. Our bound does not contradict to the existing bounds but is a complement to the existing bounds. As supported by several existing empirical results, low-rankness is not the only reason for generalization. Furthermore, our bound can be combined with the existing bounds to obtain a tighter bound. Our result sheds new light on understanding generalization of neural networks with full-rank weight matrices, and it provides a connection between operator-theoretic analysis and generalization of neural networks. 
\end{abstract}

\section{Introduction}
Understanding the generalization property has been one of the biggest topics for analyzing neural networks.
A major approach for theoretical investigation of this topic is bounding some complexity of networks~\citep{bartlett02,mohri18}.
Intuitively, a large number of parameters \red{makes} the complexity and generalization error large.
This intuition has been studied, for example, based on a classical VC-dimension theory~\citep{harvey17,anthony09}.
However, for neural networks, small generalization error can be achieved even in over-parameterized regimes~\citep{novak18,neyshabur19}.
To explain this behavior, norm-based bounds have been investigated~\citep{neyshabur15,bartlett17,golowich18,neyshabur18,wei19,wei20,li21,ju22,weinan22}.
These bounds do not depend on the number of parameters explicitly.
However, they are typically described by the $(p,q)$ norms of the weight matrices, and if the norms are large, these bounds grow exponentially with respect to the depth of the \red{network. Another} approach to tackle over-parameterized networks is a compression-based approach~\citep{arora18,Suzuki20}.
These bounds explain the generalization of networks by investigating how much the networks can be compressed.
The bounds get smaller as the ranks of the weight matrices become smaller.
However, low-rankness is not the only reason for generalization.
\citet{goldblum20} empirically showed that even with high-rank weight matrices, networks generalize well.
This implies that if the ranks of the weight matrices are large, the existing compression-based bounds are not always tight.

In this paper, we derive a completely different type of uniform bounds of complexity using Koopman operators, which sheds light on why networks generalize well even when their weights are high- or full-rank matrices.
More precisely, let $L$ be the depth, $d_j$ be the width of the $j$th layer, \red{$g$ be the final nonlinear transformation}, and $n$ be the number of samples.
For $j=1,\ldots,L$, let $W_j\in\r{d_{j}\times d_{j-1}}$ be an \red{injective} weight matrix, 
$s_j>d_j/2$ describes the smoothness of a function space $H_j$ where the Koopman operator is defined.
Our results are summarized as follows, where $\Vert\cdot\Vert$ is the operator norm, \red{$E_j$ and $G_j$ are factors determined by the activation functions and the range of $W_j$, respectively:} 
\begin{align}
\mbox{Rademacher complexity}
\red{\le O\bigg(\frac{\Vert g\Vert_{H_L} }{\sqrt{n}}\prod_{j=1}^L\frac{G_jE_j \Vert W_j\Vert^{s_{j-1}}}{\opn{det}(W_j^*W_j)^{1/4}}\bigg)}.\label{eq:main_bound}
\end{align}
%
Surprisingly, the determinant factor tells us that if the singular values of $W_j$ are large, the bound gets small.
It is tight when the condition number of $W_j$, i.e., the ratio of the largest and the smallest singular values, is small. 
Especially, when $W_j$ is orthogonal,
\red{$G_j=1$} and the factor $\Vert W_j\Vert^{s_{j-1}}/\opn{det}(W_j^*W_j)^{1/4}$ reduces to $1$.
We can interpret that $W_j$ transforms signals in certain directions, which makes it easy for the network to extract features of data.
Networks with orthogonal weight matrices have been proposed~\citep{maduranga19,wang20,li21}.
Our bound also justifies the generalization property of these networks.

In addition to providing the new perspective, we can combine our bound with existing bounds.
In other words, our bound can be a complement to existing bounds.
\citet{goldblum20} pointed out that the rank of the weight matrix tends to be large near the input layer, but be small near the output layer.
By adopting our bound for lower layers and existing bounds for higher layers, we obtain a tight bound that takes the role of each layer into account.
The determinant factors come from Koopman operators.
Koopman operator is a linear operator defined by the composition of functions.
It has been investigated for analyzing dynamical systems and time-series data generated from dynamical systems~\citep{koopman31,mezic12,kawahara16,ishikawa18,klus17,hashimoto20,giannakis20,blunton22}.
Its theoretical aspects also have been studied~\citep{das21,ikeda22-2,ikeda22,ishikawa23}. 
Connections between Koopman operators and neural networks have also been discussed.
For example, efficient learning algorithms are proposed by describing the learning dynamics of the parameters of neural networks by Koopman operators~\citep{redman22,dogra20}.
\citet{lusch17} applied neural networks to identifying eigenfunctions of Koopman operators to extract features of dynamical systems.
\citet{konishi23} applied Koopman operators to analyze equilibrium models.
On the other hand, in this paper, we represent the composition structure of neural networks using Koopman operators, and apply them to analyzing the complexity of neural networks from an operator-theoretic perspective.

Our main contributions are as follows:
\begin{itemize}[topsep=2pt,leftmargin=*,nosep]
    \item We show a new complexity bound, which involves both the norm and determinant of the weight matrices. By virtue of the determinant term, our bound gets small if the condition numbers of the weight matrices get small. Especially, it justifies the generalization property of existing networks with orthogonal weight matrices. It also provides a new perspective about why the networks with high-rank weights generalize well (Section~\ref{sec:related_works}, the paragraph after Proposition~\ref{prop:nonsingular1}, and Remark~\ref{rmk:unitary}).
    \item We can combine our bound with existing bounds to obtain a bound which describes the role of each layer (Subsection~\ref{subsec:combine}).
    \item We provide an operator-theoretic approach to analyzing networks. We use Koopman operators to derive the determinant term in our bound (Subsection~\ref{subsec:settings} and Subsection~\ref{subsec:invertible}).
\end{itemize}
We emphasize that our operator-theoretic approach reveals a new aspect of neural networks.

\section{Related works}\label{sec:related_works}
\paragraph{Norm-based bounds} Generalization bounds based on the norm of $W_j$ have been investigated in previous studies.
These bounds are described typically by the $(p,q)$ matrix norm $\Vert W_j\Vert_{p,q}$ of $W_j$ (see the upper part of Table~\ref{tab:existing_bound}).
Thus, although these bounds do not explicitly depend on the width of the layers, they tend to be large as the width of the layers becomes large.
For example, the $(p,q)$ matrix norm of the $d$ by $d$ identity matrix is $d^{1/p}$.
This situation is totally different from our case, where the factor \red{related to the weight matrix} is reduced to $1$ if it is orthogonal.
We can see our bound is described by the spectral property of the weight matrices and tight when the condition number of $W_j$ is small.
\citet{bartlett17,wei20,ju22} showed bounds with reference matrices $A_j$ and $B_j$.
These bounds allow us to explain generalization property through the discrepancy between the weight matrices and fixed reference matrices.
A main difference of these bounds from ours is that the existing bounds only focus on the discrepancy from {\em fixed reference matrices} whereas our bound is based on the discrepancy of the spectral property from a certain {\em class of matrices}, which is much broader than fixed matrices.
\citet{li21} showed a bound described by $\vert\prod_{j=1}^L\Vert W_j\Vert -1\vert$, the discrepancy between the product of largest singular values of $W_1,\ldots,W_L$ and $1$.
It is motivated by the covering number $c_{\mcl{X}}$ of the input space $\mcl{X}$ and a diameter $\gamma_{\mathbf{x}}$ of the covering set, which depends on the input samples $\mathbf{x}=(x_1,\ldots,x_n)$.
This bound somewhat describes the spectral properties of $W_j$.
However, they only focus on the largest singular values of $W_j$ and do not take the other singular values into account.
On the other hand, our bound is described by how much the whole singular values of $W_j$ differ from its largest singular value. 

\paragraph{Compression-based bounds}
\citet{arora18} and \citet{Suzuki20} focused on compression of the network and showed bounds that also get small as the rank of the weight matrices decreases, with the bias $\hat{r}$ induced by compression (see the middle part of Table~\ref{tab:existing_bound}).
\citet{arora18} introduced layer cushion $\mu_j$ and interlayer cushion $\mu_{j\rightarrow}$ for layer $j$, which tends to be small if the rank of $W_j$ is small.
They also observed that noise is filtered by the weight matrices whose stable ranks are small.
\citet{Suzuki20} showed the dependency on the rank more directly.
The bounds describe why networks with low-rank matrices generalize well.
However, networks with high-rank matrices can also generalize well~\citep{goldblum20}, and the bounds cannot describe this phenomenon.
\citet[Figure 1]{arora18} also empirically show that noise rapidly decreases on higher layers.
Since the noise stability is related to the rank of the weight matrices, the result supports the tightness of their bound on higher layers.
However, the result also implies that noise does not really decrease on lower layers, and we need an additional theory that describes what happens on lower layers.

\begin{table}[t]
    \vspace{-.8cm}
    \centering
    \caption{Comparison of our bound to existing bounds.}\vspace{-.03cm}
    \begin{tabular}{c|c|c}
        \hline
         Authors & Rate & Type\\
         \hline
         \citet{neyshabur15}& $\frac{2^L\prod_{j=1}^L\Vert W_j\Vert_{2,2}}{\sqrt{n}}\rule{0pt}{15pt}$ & \multirow{7}{*}{Norm-based}\\
         \citet{neyshabur18}&$\frac{L\max_jd_j\prod_{j=1}^L\Vert W_j\Vert}{\sqrt{n}}\Big(\sum_{j=1}^L\frac{\Vert W_j\Vert_{2,2}^2}{\Vert W_j\Vert^2}\Big)^{1/2}$ & \\
         \citet{golowich18} &$\Big(\prod_{j=1}^L\Vert W_j\Vert_{2,2}\Big)\min\bigg\{\frac1{n^{1/4}},\sqrt{\frac{L}{n}}\bigg\}$ & \\
         \citet{bartlett17}& $\frac{\prod_{j=1}^L\Vert W_j\Vert}{\sqrt{n}}\bigg(\sum_{j=1}^L\frac{\Vert W_j^T-A_j^T\Vert_{2,1}^{2/3}}{\Vert W_j\Vert^{2/3}}\bigg)^{3/2}$ & \\
         \citet{wei20}&$\frac{(\sum_{j=1}^L\kappa_j^{2/3}\min\{L^{1/2}\Vert W_j-A_j\Vert_{2,2},\Vert W_j-B_j\Vert_{1,1}\}^{2/3})^{3/2}}{\sqrt{n}}\rule{0pt}{16pt}$&\\
         \citet{ju22}&$\frac{\sum_{j=1}^L\theta_j\Vert W_j-A_j\Vert_{2,2}}{\sqrt{n}}\rule{0pt}{14pt}$& \\
         \citet{li21} & $\Vert \mathbf{x}\Vert\vert\prod_{j=1}^L\Vert W_{j}\Vert-1\vert+\gamma_{\mathbf{x}}+\sqrt{\frac{c_{\mcl{X}}}{n}}\rule{0pt}{12pt}$& \\
         \hline
         \citet{arora18}&$\hat{r}+\frac{L\max_i\Vert f(x_i)\Vert}{\hat{r}\sqrt{n}} \Big(\sum_{j=1}^L\frac{1}{\mu_j^2\mu_{j\rightarrow}^2}\Big)^{1/2}\rule{0pt}{17pt}$&\multirow{2}{*}{Compression}\\ 
         \citet{Suzuki20}& $\frac{\hat{r}}{\sqrt{n}}+\sqrt{\frac{L}{n}}\big(\sum_{j=1}^L\tilde{r}_j(\tilde{d}_{j-1}+\tilde{d}_j)\big)^{1/2}\rule{0pt}{15pt}$ & \\
         \hline
         Ours & \red{$\frac{\Vert g\Vert_{H_L} }{\sqrt{n}}\prod_{j=1}^L\frac{ G_jE_j\Vert W_j\Vert^{s_{j-1}}}{\opn{det}(W_j^*W_j)^{1/4}}\rule{0pt}{14pt}$} & Operator-theoretic\\
         \hline
    \end{tabular}
    \label{tab:existing_bound}\vspace{-.15cm}
\begin{flushleft}
\begin{spacing}{0.65}
{\footnotesize $\kappa_j$ and $\theta_j$ are determined by the 
    Jacobian and Hessian of the network $f$ with respect to the $j$th layer and $W_j$, respectively.
    $\tilde{r}_j$ and $\tilde{d}_j$ are the rank and dimension of the $j$th weight matrices for the compressed network.}
\end{spacing}
\end{flushleft}
    \vspace{-.76cm}
\end{table}

\section{Preliminaries}\label{sec:preliminaries}

\subsection{Notation} 
For a linear operator $W$ on a Hilbert space, its range and kernel are denoted by $\mcl{R}(W)$ and $\opn{ker}(W)$, respectively.
Its operator norm is denoted by $\Vert W\Vert$.
For a function $p\in L^{\infty}(\mathbb{R}^{d})$, its $L^{\infty}$-norm is denoted by $\Vert p\Vert_{\infty}$.
For a function $h$ on $\r{d}$ and a subset $\mcl{S}$ of $\r{d}$, the restriction of $h$ on $\mcl{S}$ is denoted by $h|_S$.
\red{For $f\in L^2(\r{d})$, we denote the Fourier transform of $f$ as $\hat{f}(\omega)=\int_{\mathbb{R}^d}f(x)\mr{e}^{-\mr{i}x\cdot\omega}\mr{d}x$.
We denote the adjoint of a matrix $W$ by $W^*$.}
\subsection{Koopman operator}\label{subsec:koopman}
Let $\mcl{F}_1$ and $\mcl{F}_2$ be function spaces on $\r{d_1}$ and $\r{d_2}$.
For a function $f:\r{d_1}\to\r{d_2}$, we define the Koopman operator from $\mcl{F}_2$ to $\mcl{F}_1$ by the composition as follows.
Let $\mcl{D}_f=\{g\in\mcl{F}_2\,\mid\,g\circ f\in\mcl{F}_1\}$.
The {\em Koopman operator} $K_f$ from $\mcl{F}_2$ to $\mcl{F}_1$ with respect to $f$ is defined as $K_fg=g\circ f$ for $g\in\mcl{D}_f$.

\subsection{Reproducing kernel Hilbert space (RKHS)}\label{subsec:rkhs}
As function spaces, we consider RKHSs.
Let $p$ be a non-negative function such that $p\in L^1(\mathbb{R}^d)$.
Let $H_{p}(\mathbb{R}^d)=\{f\in L^2(\r{d})\,\mid\,\hat{f}/\sqrt{p}\in L^2(\r{d})\}$ be the Hilbert space equipped with the inner product $\bracket{f,g}_{H_p(\r{d})}=\int_{\r{d}}\hat{f}(\omega)\hat{g}(\omega)/p(\omega)\mr{d}\omega$.
We can see that $k_p(x,y):=\int_{\r{d}}\mr{e}^{\mr{i}(x-y)\cdot\omega}p(\omega)\mr{d}\omega$ is the reproducing kernel of $H_p(\r{d})$, i.e., $\bracket{k_p(\cdot,x),f}_{H_p(\r{d})}=f(x)$ for $f\in H_p(\r{d})$.
Note that since $H_p(\r{d})\subseteq L^2(\r{d})$, the value of functions in $H_p(\r{d})$ vanishes at infinity.

One advantage of focusing on RKHSs is that they have well-defined evaluation operators induced by the reproducing property.
This property makes deriving an upper bound of the Rademacher complexity easier (see, for example,~\citet[Theorem 6.12]{mohri18}).
\begin{example}\label{ex:sobolev}
Set $p(\omega)=1/(1+\Vert\omega\Vert^2)^s$ for $s>d/2$.
Then, $H_p(\r{d})$ is the Sobolev space $W^{s,2}(\r{d})$ of order $s$.
Here, $\Vert \omega\Vert$ is the Euclidean norm of $\omega\in\r{d}$.
\red{Note that if $s\in\mathbb{N}$, then $\Vert f\Vert_{H_p(\r{d})}^2=\sum_{\vert\alpha\vert\le s}c_{\alpha,s,d}\Vert \partial^{\alpha}f\Vert_{L^2(\mathbb{R}^d)}^2$,
where $c_{\alpha,s,d}=(2\pi)^d s!/\alpha!/(s-\vert\alpha\vert)!$ and $\alpha$ is a multi index.}
See Appendix~\ref{ap:sobolev_norm} for more details.
\end{example}
%

\subsection{Problem setting}\label{subsec:settings}
In this paper, we consider an $L$-layer deep neural network.
Let $d_0$ be the dimension of the input space.
For $j=1,\ldots,L$, we set the width as $d_j$, let $W_j:\r{d_{j-1}}\to\r{d_j}$ be a linear operator, let $b_j:\r{d_j}\to\r{d_j}$ be a shift operator defined as $x\mapsto x+{a}_j$ with a bias ${a}_j\in\r{d_j}$, and let $\sigma_j:\r{d_j}\to\r{d_j}$ be a nonlinear activation function.
In addition, let $g:\r{d_L}\to\mathbb{C}$ be a nonlinear transformation in the final layer.
We consider a network $f$ defined as
\begin{equation}
f=g\circ b_L \circ W_L\circ \sigma_{L-1}\circ b_{L-1}\circ W_{L-1}\circ\cdots \circ \sigma_1\circ b_1\circ W_1.\label{eq:nn}
\end{equation}
Typically, we regard that $W_1$ is composed by $b_1$, $\sigma_1$ to construct the first layer.
We sequentially construct the second and third layers, and so on.
Then we get the whole network $f$.
On the other hand, from operator-theoretic perspective, we analyze $f$ in the opposite direction.
For $j=0,\ldots,L$, let $p_j$ be a density function of a finite positive measure on $\r{d_j}$.
The network $f$ is described by the Koopman operators $K_{W_j}:H_{p_{j}}(\r{d_{j}})\to H_{p_{j-1}}(\r{d_{j-1}})$, $K_{b_j},K_{\sigma_j}:H_{p_{j}}(\r{d_{j}})\to H_{p_{j}}(\r{d_{j}})$ as
\begin{equation*}
f=K_{W_1}K_{b_1}K_{\sigma_1}\cdots K_{W_{L-1}}K_{b_{L-1}}K_{\sigma_{L-1}}K_{W_L}K_{b_L}g.
\end{equation*}
The final nonlinear transformation $g$ is first provided.
Then, $g$ is composed by $b_L$ and $W_L$, i.e., the corresponding Koopman operator acts on $g$.
We sequentially apply the Koopman operators corresponding to the $(L-1)$th layer, $(L-2)$th layer, and so on.
Finally, we get the whole network~$f$.
By representing the network using the product of Koopman operators, we can bound the Rademacher complexity with the product of the norms of the Koopman operators.

To simplify the notation, we denote $H_{p_{j}}(\r{d_{j}})$ by $H_j$.
We impose the following assumptions.
\begin{assumption}\label{assum:g_sigma}
 The function $g$ is contained in $H_L$, and $K_{\sigma_j}$ is bounded for $j=1,\ldots,L$. 
\end{assumption}
\begin{assumption}\label{assum:kernel}
There exists $B>0$ such that for any $x\in\r{d}$, $\vert k_{p_0}(x,x)\vert \le B^2$.
\end{assumption}

We denote by $F$ the set of all functions in the form of \eqref{eq:nn} with Assumption~\ref{assum:g_sigma}.
\red{As a typical example, if we set $p_j(\omega)=1/(1+\Vert \omega\Vert^2)^{s_j}$ for $s_j>d_j/2$ and $g(x)=\mr{e}^{-\Vert x\Vert^2}$, then Assumption~\ref{assum:g_sigma} holds if $K_{\sigma_j}$ is bounded, and $k_{p_0}$ satisfies Assumption~\ref{assum:kernel}.}

\begin{remark}\label{rmk:final_trans}
\red{Let $g$ be a smooth function which does not decay at infinity, (e.g., sigmoid)}.
Although $H_p(\r{d})$ does not contain $g$, we can construct a function $\tilde{g}\in H_p(\r{d})$ such that $\tilde{g}(x)=g(x)$ for $x$ in a sufficiently large compact region and replace $g$ by $\tilde{g}$ in practical cases.
See Remark~\ref{rmk:bump_func} for details.
\end{remark}

For the activation function, we have the following proposition (c.f. \citet[Theorem 4.46]{sawano18}).
\begin{proposition}\label{prop:activation_koopman}
Let $p(\omega)=1/(1+\Vert \omega\Vert^2)^s$ for $\omega\in\r{d}$ $s\in\mathbb{N}$, and $s>d/2$.
If the activation function $\sigma$ has the following properties, then $K_{\sigma}:H_p(\r{d})\to H_p(\r{d})$ is bounded.
\begin{itemize}[nosep,leftmargin=*]
    \item $\sigma$ is $s$-times differentiable and its derivative $\partial^{\alpha}\sigma$ is bounded for any multi-index $\alpha\in\{(\alpha_1,\ldots,\alpha_d)\,\mid\,\sum_{j=1}^d\alpha_j\le s\}$.
    \item $\sigma$ is bi-Lipschitz, i.e., $\sigma$ is bijective and both $\sigma$ and $\sigma^{-1}$ are Lipschitz continuous.
\end{itemize}
\end{proposition}
\begin{example}
We can choose $\sigma$ as a smooth version of Leaky ReLU~\citep{biswas22}.
We will see how we can deal with other activation functions, such as the sigmoid and the hyperbolic tangent in Remark~\ref{rmk:activation_weighted}.
\end{example}
\begin{remark}\label{rmk:k_sigma}
 We have $\Vert K_{\sigma}\Vert\le \Vert \opn{det}(J\sigma^{-1})\Vert_{\infty}\max\{1,\Vert \partial_1\sigma\Vert_{\infty},\ldots,\Vert\partial_d\sigma\Vert_{\infty}\}$ if $s=1$ and $\sigma$ is elementwise, where $J\sigma^{-1}$ is the Jacobian of $\sigma^{-1}$.
 \red{As we will discuss in Appendix~\ref{ap:k_sigma}, deriving a tight bound for a larger $s$ is challenging and future work.}
\end{remark}



\section{Koopman-based bound of Rademacher complexity}\label{sec:bound}
We derive an upper bound of the Rademacher complexity.
We first focus on the case where the weight matrices are invertible or injective.
Then, we generalize the results to the non-injective case.

Let $\Omega$ be a probability space equipped with a probability measure $P$.
We denote the integral $\int_{\Omega}s(\omega)\mr{d}P(\omega)$ of a measurable function $s$ on $\Omega$ by $\mr{E}[s]$.
Let $s_1,\ldots,s_n$ be i.i.d. Rademacher variables.
For a function class $G$ and $x_1,\ldots,x_n\in\r{d}$, we denote the empirical Rademacher complexity by $\hat{R}_n(\mathbf{x},G)$,
where $\mathbf{x}=(x_1,\ldots,x_n)$.
We will provide an upper bound of $\hat{R}_n(\mathbf{x},G)$ using Koopman operators in the following subsections.
\subsection{Bound for invertible weight matrices ($d_j=d$)}\label{subsec:invertible}
In this subsection, we focus on the case $d_j=d\ (j=0,\ldots,L)$ for some $d\in\mathbb{N}$ and $W_j$ is invertible for $j=1,\ldots,L$.
This is the most fundamental case.
For $C,D>0$, set a class of weight matrices $\mcl{W}(C,D)=\{W\in\r{d\times d}\,\mid\,\Vert W\Vert\le C,\ \vert\opn{det}(W)\vert\ge D\}$ and a function class ${F}_{\opn{inv}}(C,D)=\{f\in F\,\mid\,W_j\in\mcl{W}(C,D)\}$.
We have the following theorem for a bound of Rademacher complexity. 
\begin{theorem}[First Main Theorem]\label{thm:nonsingular}
The Rademacher complexity $\hat{R}_n(\mathbf{x},{F}_{\opn{inv}}(C,D))$ is bounded as 
\begin{align}
\hat{R}_n(\mathbf{x},{F}_{\opn{inv}}(C,D))\le \frac{B\Vert g\Vert_{H_L}}{\sqrt{n}} \!\!\!\!\!\sup_{W_j\in\mcl{W}(C,D)}\bigg(\prod_{j=1}^L\frac{\Vert {p_{j}}/({p_{j-1}\circ W_j^*)}\Vert_{\infty}^{1/2}}{\vert \opn{det}(W_j)\vert^{1/2}}\bigg)\bigg(\prod_{j=1}^{L-1}\Vert K_{\sigma_j}\Vert\bigg).\label{eq:bound_invertible}
\end{align}
\end{theorem}
By representing the network using the product of Koopman operators, we can get the whole bound by bounding the norm of each Koopman operator.
A main difference of our bound from existing bounds, such as the ones by \citet{neyshabur15,golowich18} is that our bound has the determinant factors in the denominator in Eq.~\eqref{eq:bound_invertible}.
They come from a change of variable when we bound the norm of the Koopman operators, described by the following inequality:
\begin{align*}
\Vert K_{W_j}\Vert\le \big(\big\Vert{p_{j}}/{p_{j-1}\circ W_j^*}\big\Vert_{\infty}/{\vert\opn{det}(W_j)\vert}\big)^{1/2},\quad
\Vert K_{b_j}\Vert=1
\end{align*}
for $j=1,\ldots,L$.
Since the network $f$ in Eq.~\eqref{eq:nn} satisfies $f\in H_0$, using the reproducing property of $H_0$, we derive an upper bound of $\hat{R}_n(\mathbf{x},F(C,D))$ in a similar manner to that for kernel methods (see, \citet[Theorem 6.12]{mohri18}).

Regarding the factor $\Vert {p_{j}}/(p_{j-1}\circ W_j^*)\Vert_{\infty}$ in Eq.~\eqref{eq:bound_invertible}, we obtain the following lemma and proposition, which show it is bounded by $\Vert W_j\Vert$ and induces the factor $\Vert W_j\Vert^{s_{j-1}}/\opn{det}(W_j^*W_j)^{1/4}$ in Eq.~\eqref{eq:main_bound}.
\begin{lemma}\label{lem:Wnorm}
Let $p(\omega)=1/(1+\Vert\omega\Vert^2)^s$ for $s>d/2$ and $p_j=p$ for $j=0,\ldots,L$.
Then, we have
$\Vert{p}/({p\circ W_j^*})\Vert_{\infty}\le {\max\{1,{\Vert W_j\Vert^{2s}}\}}$.
\end{lemma}
As a result, the following proposition is obtained by applying Lemma~\ref{lem:Wnorm} to Theorem~\ref{thm:nonsingular}.
\begin{proposition}\label{prop:nonsingular1}
Let $p(\omega)=1/(1+\Vert\omega\Vert^2)^s$ for $s>d/2$ and $p_j=p$ for $j=0,\ldots,L$.
We have 
\begin{align*}
\hat{R}_n(\mathbf{x},{F}_{\opn{inv}}(C,D))\le \frac{B\Vert g\Vert_{H_L}}{\sqrt{n}} \bigg(\frac{\max\{1,C^{s}\}}{\sqrt{D}}\bigg)^L\bigg(\prod_{j=1}^{L-1}\Vert K_{\sigma_j}\Vert\bigg).
\end{align*}
\end{proposition}
%
\paragraph{Comparison to existing bounds}
We investigate the bound~\eqref{eq:main_bound} in terms of the singular values of $W_j$.
Let $\eta_{1,j}\ge\ldots\ge\eta_{d,j}$ be the singular values of $W_j$ and let $\alpha=s-d/2$.
Then, the term depending on the weight matrices in bound~\eqref{eq:main_bound} is described as $\eta_{1,j}^{\alpha}\prod_{i=1}^dr_{i,j}^{1/2}$, where $r_{i,j}=\eta_{1,j}/\eta_{i,j}$.
On the other hand, the existing bound by~\citet{golowich18} is described as $\eta_{1,j}(\sum_{i=1}^dr_{i,j}^{-2})^{1/2}$.
Since they are just upper bounds, our bound does not contradict the existing bound.
Our bound is tight when the condition number $r_{d,j}$ of $W_j$ is small.
The existing bound is tight when $r_{d,j}$ is large.
Note that the factors in our bound~\eqref{eq:main_bound} are bounded below by $1$.


\begin{remark}\label{rmk:unitary}
If a weight matrix $W_j$ is orthogonal, then the factor ${ \Vert W_j\Vert^{s_{j-1}}}/{\opn{det}(W_j)^{1/2}}$ in the bound~\eqref{eq:main_bound} is reduced to $1$.
On the other hand, existing norm-based bounds are described by the $(p.q)$ matrix norm.
The $(p,q)$ matrix norm of the $d$ by $d$ identity matrix is $d^{1/p}$, which is larger than $1$.
Indeed, limiting the weight matrices to orthogonal matrices has been proposed~\citep{maduranga19,wang20,li21}.
The tightness of our bound in the case of orthogonal matrices implies the advantage of these methods from the perspective of generalization.
\end{remark}

\begin{remark}
There is a tradeoff between $s$ and $B$.
We focus on the case where $p_j(\omega)=1/(1+\Vert\omega\Vert^2)^s$.
According to Lemma~\ref{lem:Wnorm}, the factor $\Vert {p_{j}}/(p_{j-1}\circ W_j^*)\Vert_{\infty}$ becomes small as $s$ approaches to $d/2$.
However, the factor $B$ goes to infinity as $s$ goes to $d/2$.
Indeed, we have $k_p(x,x)=\int_{\r{d}}p(\omega)\mr{d}\omega= C\int_{0}^{\infty}\frac{r^{d-1}}{(1+r^2)^s}\mr{d}r\ge C\int_1^\infty \frac{(r-1)^{d/2 -1}}{2r^s} dr > \frac{C}{4}\int_2^\infty r^{d/2-1-s} dr = \frac{C}{4} \frac{2^{d/2-s}}{s-d/2}$ for some constant $C>0$.
This behavior corresponds to the fact that if $s=d/2$, the Sobolev space is not an RKHS, and the evaluation operator becomes unbounded.
\end{remark}

\subsection{Bound for injective weight matrices ($d_j\ge d_{j-1}$)}
We generalize Theorem~\ref{thm:nonsingular} to that for injective 
weight matrices.
If $d_j>d_{j-1}$, then $W_j$ is not surjective.
However, it can be injective, and we first focus on the injective case.
Let $\mcl{W}_j(C,D)=\{W\in\r{d_{j-1}\times d_{j}}\,\mid\,d_j\ge d_{j-1},\ \Vert W\Vert\le C,\ \sqrt{\opn{det}({W^*W})}\ge D\}$ and $F_{\opn{inj}}(C,D)=\{f\in F\,\mid\, W_j\in\mcl{W}_j(C,D)\}$.
Let $f_j=g\circ b_L \circ W_L\circ \sigma_{L-1}\circ b_{L-1}\circ W_{L-1}\circ\cdots \circ \sigma_j\circ b_j$ and $G_j=\Vert f_j|_{\mcl{R}(W_j)}\Vert_{H_{\red{p_{j-1}}}(\mcl{R}(W_j))}/\Vert f_j\Vert_{H_{j}}$.
We have the following theorem for injective weight matrices.
\begin{theorem}[Second Main Theorem]~\label{thm:injective}
The Rademacher complexity $\hat{R}_n(\mathbf{x},F_{\opn{inj}}(C,D))$ is bounded as 
\begin{align}
\hat{R}_n(\mathbf{x},F_{\opn{inj}}(C,D)))\le \frac{B\Vert g\Vert_{H_L}}{\sqrt{n}} \!\!\!\!\!\sup_{W_j\in\mcl{W}_j(C,D)}\!\!\bigg(\prod_{j=1}^L\frac{\Vert {\red{p_{j-1}}}/({p_{j-1}\circ W_j^*})\Vert_{\mcl{R}(W_j),\infty}^{1/2}G_j}{{\opn{det}({W_j^*W_j)}}^{1/4}}\bigg)\bigg(\prod_{j=1}^{L-1}\Vert K_{\sigma_j}\Vert\bigg).\label{eq:bound_injectve}
\end{align}
\end{theorem}
Since $W_j$ is not always square, we do not have $\opn{det}(W_j)$ in Theorem~\ref{thm:nonsingular}.
However, we can replace $\opn{det}(W_j)$ with $\opn{det}(W_j^*W_j)^{1/2}$.
As Lemma~\ref{lem:Kw_norm}, we have the following lemma about the factor $\Vert {p_{j}}/({p_{j-1}\circ W_j^*})\Vert_{\mcl{R}(W_j),\infty}$ in Eq.~\eqref{eq:bound_injectve}.
%

\begin{lemma}\label{lem:Wnorm_injective}
Let $p_j(\omega)=1/(1+\Vert\omega\Vert^2)^{s_j}$ for $s_j>d_j/2$ and for $j=0,\ldots,L$.
Then, we have
\begin{align*}
\bigg\Vert \frac{p_{j-1}}{p_{j-1}\circ W_j^*}\bigg\Vert_{\mcl{R}(W_j),\infty}\le \max\{1,{\Vert W_j\Vert^{2s_{j-1}}}\}.
\end{align*}
\end{lemma}

Applying Lemma~\ref{lem:Wnorm_injective} to Theorem~\ref{thm:injective}, we finally obtain the bound~\eqref{eq:main_bound}.
Regarding the factor $G_j$, the following lemma shows that $G_j$ is determined by the isotropy of $f_j$.
\begin{lemma}\label{lem:G_j}
\red{Let $p_j(\omega)=1/(1+\Vert\omega\Vert^2)^{s_j}$ for $s_j>d_j/2$ and for $j=0,\ldots,L$.
Let $s_j\ge s_{j-1}$ and $\tilde{H}_j$ be the Sobolev space on $\mcl{R}(W_j)^{\perp}$ with $p(\omega)=1/(1+\Vert\omega\Vert^2)^{s_j-s_{j-1}}$.}
Then, $G_j\le G_{j,0}\Vert f_j|_{\mcl{R}(W_j)}\cdot f_j|_{\mcl{R}(W_j)^{\perp}}\Vert_{H_j}/\Vert f_j\Vert_{H_j}$, where 
$G_{j,0}=\Vert f_j\Vert_{\red{\tilde{H}_j}}^{-1}$.
\end{lemma}
\begin{remark}\label{rmk:G_j}
\red{The factor $G_j$ is expected to be small.
Indeed, if $f_j(x)=\mr{e}^{-c\Vert x\Vert^2}$ (It is true if $g$ is Gaussian and $j=L$.), $G_j$ gets small as $s_j$ and $d_j$ gets large.
Moreover, $G_j$ can alleviate the dependency of $\Vert K_{\sigma_j}\Vert$ on $d_j$ and $s_j$.
See Appendix~\ref{ap:G_j} for more details.}
\end{remark}

\subsection{Bounds for non-injective weight matrices}\label{subsec:non-invertible}
The bounds in Theorems~\ref{thm:nonsingular} and \ref{thm:injective} are only valid for injective weight matrices, and the bound goes to infinity if they become singular. 
This is because if $W_j:\r{d_{j-1}}\to\r{d_j}$ has singularity, $h\circ W$ for $h\in H_j$ becomes constant along the direction of $\opn{ker}(W_j)$.
As a result, $h\circ W_j$ \red{is} not contained in $H_{j-1}$ and $K_{W_j}$ becomes unbounded.
To deal with this situation, we propose two approaches that are valid for non-injective weight matrices, graph-based and weighted Koopman-based ones.
\subsubsection{Graph-based bound}\label{subsec:injective_based}
The first approach to deal with this situation is constructing an injective operator related to the graph of $W_j$.
For $j=1,\ldots,L$, let \red{$r_j=\opn{dim}(\opn{ker}(W_j))$, $\delta_j=\sum_{k=0}^jr_k$}, and $\tilde{W}_j$ be defined as \red{$\tilde{W}_j(x_1,x_2)=(W_jx_1,P_jx_1,x_2)$ for $x_1\in\r{d_{j-1}}$ and $x_2\in\r{\delta_{j-2}}$ for $j\ge 2$ and $\tilde{W}_jx=(W_jx,P_jx)$ for $j=1$.
Here, $P_j$ is the projection onto $\opn{ker}(W_j)$.}
Then, $\tilde{W}_j$ is injective (See Appendix~\ref{ap:injective}). 
Let $\tilde{\sigma}_j:\r{\delta_j}\to\r{\delta_j}$ and $\tilde{b}_j:\r{\delta_j}\to\r{\delta_j}$ be defined respectively as $\tilde{\sigma}_j(x_1,x_2)=(\sigma_j(x_1),x_2)$ and $\tilde{b}_j(x_1,x_2)=(b_j(x_1),x_2)$ for $x_1\in\r{d_{j}}$ and $x_2\in\r{\delta_{j-1}}$.
In addition, let $\tilde{g}\in H_{p_L}(\r{\delta_L})$ be defined as $\tilde{g}(x_1,x_2)=g(x_1)\psi(x_2)$ for $x_1\in\r{d_L}$ and $x_2\in\r{\delta_{L-1}}$, where $\psi$ is a rapidly decaying function on $\r{\delta_{L-1}}$ so that $\tilde{g}\in H_{p_L}(\r{\delta_{L-1}+d_L})$.
Consider the following network:
\begin{align}
\tilde{f}(x)&=\tilde{g}\circ \tilde{b}_L \circ \tilde{W}_L\circ \cdots \circ \tilde{\sigma}_1\circ \tilde{b}_1\circ \tilde{W}_1(x)\nn\\
&=\psi(\red{P_{L}}\sigma_{L-1}\circ b_{L-1} \circ W_{L-1}\circ \cdots \circ \sigma_1\circ b_1\circ W_1(x),\ldots,\red{P_2}\sigma_1\circ b_1\circ W_1(x),\red{P_1}x)\nn\\
&\qquad\qquad\cdot g(b_{L} \circ W_{L}\circ \cdots \circ \sigma_1\circ b_1\circ W_1(x)).\label{eq:nn_mod}
\end{align}
Since $\opn{det}(\tilde{W}_j^*\tilde{W}_j)=\opn{det}(W_j^*W_j+I)$ and $\Vert\tilde{W}_j\Vert=\sqrt{\Vert I+W_j^*W_j\Vert}$,
we set $\tilde{\mcl{W}}_j(C,D)=\{{W}\in\r{d_{j-1}\times d_{j}}\,\mid\,\sqrt{\Vert I+W^*W\Vert}\le C,\ \sqrt{\opn{det}(W^*W+I)}\ge D\}$ and $\tilde{F}(C,D)=\{\tilde{f}\,\mid\,f\in F,\ {W}_j\in\tilde{\mcl{W}}_j(C,D)\}$.
Moreover, put $\tilde{H}_L=H_{p_L}(\r{\delta_L})$. 
 By Theorem~\ref{thm:injective}, we obtain the following bound, where the determinant factor does not go to infinity by virtue of the identity $I$ appearing in $\tilde{W}_j$.
\begin{proposition}~\label{prop:graph}
The Rademacher complexity $\hat{R}_n(\mathbf{x},\tilde{F}(C,D))$ is bounded as 
\begin{align*}
\hat{R}_n(\mathbf{x},\tilde{F}(C,D)))\le \frac{B\Vert \tilde{g}\Vert_{\tilde{H}_{L}}}{\sqrt{n}} \sup_{{W}_j\in\mcl{W}_j(C,D)}\bigg(\prod_{j=1}^L\frac{\Vert {p_{j}}/({p_{j-1}\circ \tilde{W}_j^*})\Vert_{\mcl{R}(\tilde{W}_j),\infty}^{1/2}\tilde{G}_j}{{\opn{det}({W_j^*W_j+I)}}^{1/4}}\bigg)\bigg(\prod_{j=1}^{L-1}\Vert K_{\sigma_j}\Vert\bigg).
\end{align*}
\end{proposition}

\begin{remark}\label{rmk:bump_func}
The difference between the networks \eqref{eq:nn_mod} and \eqref{eq:nn} is the factor $\psi$.
If we set $\psi$ as $\psi(x)=1$ for $x$ in a sufficiently large region $\Omega$, then $\tilde{f}(x)=f(x)$ for $x\in\Omega$.
We can set $\psi$, for example, as a smooth bump function~\citep[Chapter 13]{tu08}.
\red{See Appendix~\ref{ap:bump_function} for the definition.}
If the data is in a compact region, then the output of each layer is also in a compact region since $\Vert W_j\Vert\le \sqrt{\Vert I+W_j^*W_j\Vert}\le C$.
Thus, it is natural to assume that $f$ can be replaced by $\tilde{f}$ in practical cases.
\end{remark}
\begin{remark}\label{rmk:g_grow}
The factor $\Vert \tilde{g}\Vert_{\tilde{H}_L}$ grows as $\Omega$ becomes large.
This is because $\Vert \psi\Vert^2_{L^2(\r{\delta_{L-1}})}$ becomes large as the volume of $\Omega$ gets large.
This fact does not contradict the fact that the bound in Theorem~\ref{thm:nonsingular} goes to infinity if $W_j$ is singular.
More details are explained in Appendix~\ref{ap:g_grow}.
\end{remark}

\subsubsection{Weighted Koopman-based bound}\label{subsec:weighted_Koopman}
Instead of constructing the injective operators in Subsection~\ref{subsec:injective_based}, we can also use weighted Koopman operators.
For $\psi_j:\r{d_{j}}\to\mathbb{C}$, define $\tilde{K}_{W_j}h=\psi_j\cdot h\circ W_j$, which is called the weighted Koopman operator.
We construct the same network as $\tilde{f}$ in Eq.~\eqref{eq:nn_mod} using $\tilde{K}_{W_j}$.
\begin{align*}
\tilde{f}(x)&=\tilde{K}_{W_1}K_{b_1}K_{\sigma_1}\cdots \tilde{K}_{W_L}K_{b_L}g(x)\\
&=\psi_1(x)\psi_2(\sigma_1\circ b_1\circ W_1(x))\cdots\psi_{L}(\sigma_{L-1}\circ b_{L-1} \circ W_{L-1}\circ \cdots \circ \sigma_1\circ b_1\circ W_1(x))\\
&\qquad \qquad \cdot g(b_{L} \circ W_{L}\circ \cdots \circ \sigma_1\circ b_1\circ W_1(x))\\
&=\psi(\sigma_{L-1}\circ b_{L-1} \circ W_{L-1}\circ \cdots \circ \sigma_1\circ b_1\circ W_1(x),\ldots,\sigma_1\circ b_1\circ W_1(x),x)\nn\\
&\qquad\qquad\cdot g(b_{L} \circ W_{L}\circ \cdots \circ \sigma_1\circ b_1\circ W_1(x)),
\end{align*}
where $\psi(x_1,\ldots,x_L)=\psi_1(x_1)\cdots\psi_L(x_L)$ for $x_j\in\r{d_{j-1}}$.
Let $W_{\opn{r},j}=W_j|_{\opn{ker}(W_j)^{\perp}}$ be the restricted operator of $W_j$ to $\opn{ker}(W_j)^{\perp}$.
Set $\mcl{W}_{\opn{r},j}(C,D)=\{W\in\r{d_{j-1}\times d_{j}}\,\mid\,\Vert W\Vert\le C,\ \vert\opn{det}(W_{\opn{r}})\vert\ge D\}$
and $F_{\mr{r}}(C,D)=\{f\in F\,\mid\,{W}_{j}\in{\mcl{W}}_{\mr{r},j}(C,D)\}$.
By letting $\psi_j$ decay in the direction of $\opn{ker}(W_j)$, e.g., the smooth bump function, $\tilde{K}_{W_j}h$ for $h\in H_j$ decays in all the directions. We only need to care about $W_{\opn{r},j}$, not \red{the whole of $W_j$}.
Note that $\psi_j$ has the same role as $\psi$ in Eq.~\eqref{eq:nn_mod}.
If the data is in a compact region, we replace $f$ by $\tilde{f}$, which coincides with $f$ on the compact region.
By virtue of the decay property of $\psi_j$, we can bound $\tilde{K}_{W_j}$ even if $W_j$ is singular.
\begin{proposition}\label{prop:weighted_koopman}
Let $p_j(\omega)=1/(1+\Vert\omega\Vert^2)^{s_j}$ with $s_j\in\mathbb{N}$, $s_j\ge s_{j-1}$, and $s_j>d_j/2$.
Let $\tilde{H}_{j-1}=H_{p_{j-1}}(\opn{ker}(W_j))$ and $\psi_j$ be a function satisfying $\psi_j(x)=\psi_{j,1}(x_1)$ for some $\psi_{j,1}\in \tilde{H}_{j-1}$, where $x=x_1+x_2$ for $x_1\in\opn{ker}(W_j)$ and $x_2\in\opn{ker}(W_j)^{\perp}$.
\red{Moreover, let $G_j=\Vert f_j\Vert_{H_{j-1}(\opn{ker}(W_j)^{\perp})}/\Vert f_j\Vert_{H_j}$.}
Then, we have
\begin{align*}
\red{\hat{R}_n(\mathbf{x},{F}_{\mr{r}}(C,D)))\le \frac{B\Vert {g}\Vert_{H_L}}{\sqrt{n}}\!\!\!\!\!\!\sup_{{W}_j\in\mcl{W}_{\opn{r},j}(C,D)}\!\!\bigg(\prod_{j=1}^L\frac{\Vert \psi_{j,1}\Vert_{\tilde{H}_{j-1}}\!\!G_j\max\{1,\Vert W_j\Vert^{s_{j-1}}\}}{\vert\opn{det}(W_{{\opn{r},j}})\vert^{1/2}}\bigg)\bigg(\prod_{j=1}^{L-1}\Vert K_{\sigma_j}\Vert\bigg).}
\end{align*}
\end{proposition}
%
%


\begin{remark}\label{rmk:activation_weighted}
Although we focus on the singularity of $K_{W_j}$ and use the weighted Koopman operator with respect to $W_j$, we can deal with the singularity of $K_{\sigma_j}$ in the same manner as $K_{W_j}$, i.e., by constructing the weighted Koopman operator $\tilde{K}_{\sigma_j}$ with respect to $\sigma_j$.
For example, the sigmoid and hyperbolic tangent do not satisfy the assumption for $\sigma_j$ stated in Proposition~\ref{prop:activation_koopman}.
This is because the Jacobian of $\sigma^{-1}$ is not bounded.
However, $\tilde{K}_{\sigma_j}$ is bounded by virtue of $\psi_j$.
\end{remark}

\begin{remark}\label{rmk:psi_norm}
\red{The norm of $\psi_j$ can be canceled by the factor $G_j$.
See Appendix~\ref{ap:psi_norm} for more details.}
\end{remark}

\subsection{Combining the Koopman-based bound with other bounds}\label{subsec:combine}
In this subsection, we show that our Koopman-based bound is flexible enough to be combined with another bound.
As stated in Subsection~\ref{subsec:invertible}, the case where our bound is tight differs from the case where existing bounds such as the one by \citet{golowich18} are tight.
We can combine these bounds to obtain a tighter bound.
For $1\le l\le L$, let $F_{1:l}$ be the set of all functions in the form 
\begin{equation}
\sigma_l\circ b_l \circ W_l\circ \sigma_{l-1}\circ b_{l-1}\circ W_{l-1}\circ\cdots \circ \sigma_1\circ b_1\circ W_1\label{eq:nn_middle}
\end{equation}
with Assumption~\ref{assum:g_sigma}, and let $F_{1:l,\opn{inj}}(C,D)=\{f\in F_{1:l}\,\mid\, W_j\in\mcl{W}_j(C,D)\}$.
For $l\le L-1$, consider the set of all functions in $H_l$ which \red{have} the form 
\begin{equation*}
g\circ b_L\circ W_L\circ \sigma_{L-1}\circ b_{L-1}\circ W_{L-1}\circ\cdots \circ \sigma_{l+1}\circ b_{l+1}\circ W_{l+1}
\end{equation*}
and consider any nonempty subset $F_{l+1:L}$ of it.
For $l=L$, we set $F_{L+1:L}=\{g\}$.
Let $F_{l,\opn{comb}}(C,D)=\{f_1\circ f_2\,\mid\,f_1\in F_{l+1,L},\ f_2\in F_{1:l,\opn{inj}}(C,D)\}$.
The following proposition shows the connection between the Rademacher complexity of $F_{l,\opn{comb}}(C,D)$ and that of $F_{l+1:L}$.
\begin{proposition}\label{prop:combine}
Let $\tilde{\mathbf{x}}=(\tilde{x}_1,\ldots,\tilde{x}_n)\in(\r{d_l})^n$.
Let $v_n(\omega)=\sum_{i=1}^ns_i(\omega)k_{p_0}(\cdot,x_i)$, $\tilde{v}_n(\omega)=\sum_{i=1}^ns_i(\omega)k_{p_l}(\cdot,\tilde{x}_i)$, 
$\mcl{W}=\{(W_1,\ldots,W_l)\mid W_j\in\mcl{W}_j(C,D)\}$,
and $\gamma_n=\Vert v_n\Vert_{H_0}/\Vert \tilde{v}_n\Vert_{H_l}$.
Then, 
\begin{align*}
\hat{R}_n(\mathbf{x},F_{l,\opn{comb}}(C,D))
&\le\sup_{(W_1,\ldots,W_l)\in\mcl{W}}
\prod_{j=1}^l\bigg\Vert \frac{p_{j-1}}{p_{j-1}\circ W_j^*}\bigg\Vert_{\mcl{R}(W_j),\infty}^{1/2}\frac{G_j\Vert K_{\sigma_j}\Vert}{{\opn{det}({W_j^*W_j)}}^{1/4}}\\
&\qquad\cdot\bigg(\hat{R}_n(\tilde{\mathbf{x}},F_{l+1:L})+\frac{B}{\sqrt{n}}\inf_{h_1\in F_{l+1:L}}\!\mr{E}^{\frac12}\bigg[\!\sup_{h_2\in F_{l+1:L}}\bigg\Vert h_1-\frac{\Vert h_2\Vert_{H_l}\gamma_n}{\Vert \tilde{v}_n\Vert_{H_l}}\tilde{v}_n\bigg\Vert_{H_l}^2\bigg]\bigg).
\end{align*}
\end{proposition}
The complexity of the whole network is decomposed into the Koopman-based bound for the first $l$ layers and the complexity of the remaining $L-l$ layers, together with a term describing the approximation power of the function class corresponding to $L-l$ layers.
Note that Proposition~\ref{prop:combine} generalizes Theorem~\ref{thm:injective} up to the multiplication of a constant.
Indeed, if $l=L$, we have $\hat{R}_n(\mathbf{x},F_{l+1,L})=0$.
In addition, we have
$\inf_{h_1\in F_{l+1:L}}\mr{E}^{\frac12}[\sup_{h_2\in F_{l+1:L}}\Vert h_1-{\Vert h_2\Vert\gamma_n \tilde{v}_n}/{\Vert \tilde{v}_n\Vert}\Vert^2]
=\mr{E}^{\frac12}[\Vert g-{\Vert g\Vert \gamma_n\tilde{v}_n}/{\Vert \tilde{v}_n\Vert}\Vert^2]\le \Vert g\Vert\mr{E}^{\frac12}[(1+\gamma_n)^2]$.

\begin{remark}\label{rmk:combine}
We can also combine our Koopman-based approach with the existing ``peeling'' approach, e.g., by \citet{neyshabur15,golowich18} (see Appendix~\ref{ap:combine}).
Then, for $1\le l\le L$, we obtain a bound such as
\begin{align*}
 O\bigg(\bigg(\prod_{j=l+1}^{L}\Vert W_j\Vert_{2,2}\bigg)\bigg(\prod_{j=1}^{l}\frac{\Vert W_j\Vert^{s_j}}{\opn{det}(W_j^*W_j)^{1/4}}\bigg)\bigg).   
\end{align*} 
Typically, in many networks, the width grows sequentially near the input layer, i.e., $d_{j-1}\le d_j$ for small $j$ and decays near the output layer, i.e., $d_{j-1}\ge d_j$ for large $j$.
Therefore, this type of combination is suitable for many practical cases for deriving a tighter bound than existing bounds.
\end{remark}

Proposition~\ref{prop:combine} and Remark~\ref{rmk:combine} theoretically implies that our Koopman-based bound is suitable for lower layers.
Practically, we can interpret that signals are transformed on the lower layers so that its essential information is extracted on the higher layers.
This interpretation also supports the result in Figure 1 by~\cite{arora18}.
Noise is removed (i.e., signals are extracted) by the higher layers, but it is not completely removed by lower layers.
We will investigate the behavior of each layer numerically in Section~\ref{sec:numerical_results} and Appendix~\ref{subsec:transformation}.


\section{Numerical results}\label{sec:numerical_results}
\paragraph{Validity of the bound}
To investigate our bound numerically, we consider a regression problem on $\r{3}$, where the target function $t$ is $t(x)=\mr{e}^{-\Vert 2x-1\Vert^2}$.
We constructed a simple network $f(x)=g(W_2\sigma (W_1x+b_1)+b_2)$, where $W_1\in\r{3\times 3}$, $W_2\in\r{6\times 3}$, $b_1\in\r{3}$, $b_2\in\r{6}$, $g(x)=\mr{e}^{-\Vert x\Vert^2}$, and $\sigma$ is a smooth version of Leaky ReLU proposed by~\citet{biswas22}. 
We created a training dataset from samples randomly drawn from the standard normal distribution.
Figure~\ref{fig:mainfig} (a) illustrates the relationship between the generalization error and our bound $O(\prod_{j=1}^L\Vert W_j\Vert^{s_j}/(\mathrm{det}(W_j^*W_j)^{1/4}))$.
Here, we set $s_j=(d_j+0.1)/2$.
In Figure \ref{fig:mainfig} (a), we can see that our bound gets smaller in proportion to the generalization error.
In addition, we investigated the generalization property of a network with a regularization based on our bound.
We considered the classification task with MNIST.
For training the network, we used only $n=1000$ samples to create a situation where the model is hard to generalize.
We constructed a network with four dense layers and \red{trained} it with and without a regularization term $\Vert W_j\Vert+1/\opn{det}(I +W_j^*W_j)$, which makes both the norm and determinant of $W_j$ small.
Figure~\ref{fig:mainfig} (b) shows the test accuracy.
We can see that the regularization based on our bound leads to better generalization property, which implies the validity of our bound.
\begin{figure}[t]
    \vspace{-.8cm}
    \centering
    \begin{minipage}{0.03\textwidth}
    (a)
    \end{minipage}%
    \begin{minipage}{0.22\textwidth}
    \includegraphics[scale=0.27]{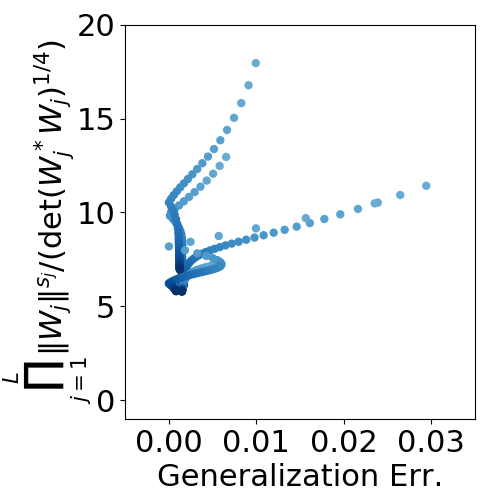} 
    \end{minipage}%
    \begin{minipage}{0.03\textwidth}
    (b)
    \end{minipage}%
    \begin{minipage}{0.34\textwidth}
    \includegraphics[scale=0.27]{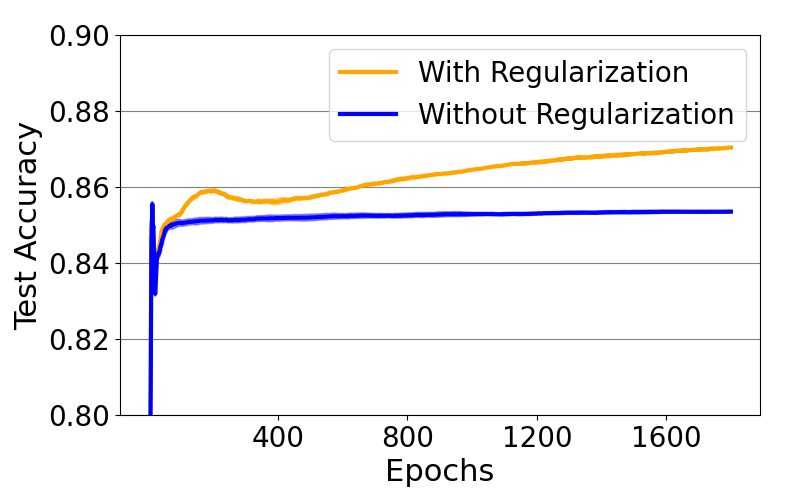} 
    \end{minipage}%
    \begin{minipage}{0.034\textwidth}
    (c)
    \end{minipage}%
    \begin{minipage}{0.3\textwidth}
    \includegraphics[scale=0.27]{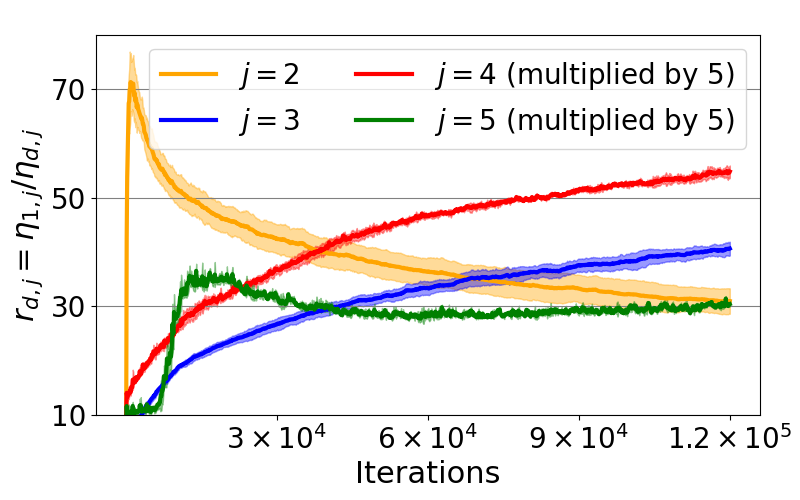}
    \end{minipage}\vspace{-.3cm}
    \caption{\small (a) Scatter plot of the generalization error versus our bound (for 5 independent runs). The color is set to get dark as the epoch proceeds. (b) Test accuracy with and without the regularization based on our bound. (c) The condition number $r_{d,j}=\eta_{1,j}/\eta_{d,j}$ of the weight matrix for layer $j=2,\ldots,4$.}\label{fig:mainfig}
\end{figure}\vspace{-.2cm}
\paragraph{Singular values of the weight matrices}
We investigated the difference in the behavior of singular values of the weight matrix for each layer.
We considered the classification task with CIFAR-10 and AlexNet~\citep{krizhevsky12}.
AlexNet has five convolutional layers followed by dense layers.
For each $j=2,\ldots,5$, we computed the condition number $r_{d,j}$ of the weight matrix.
The results are illustrated in Figure~\ref{fig:mainfig} (c).
We scaled the values for $j=4,5$ for readability.
Since the weight matrix of the first layer ($j=1$) is huge and the computational cost of computing its singular values is expensive, we focus on $j=2,\ldots,5$.
We can see that for the second layer ($j=2$), $r_{d,j}$ tends to be small as the learning process proceeds (as the test accuracy grows).
On the other hand, for the third and fourth layers ($j=3,4$), the $r_{d,j}$ tends to be large.
This means that the behavior of the singular values is different depending on the layers.
According to the paragraph after Proposition~\ref{prop:nonsingular1}, our bound becomes smaller as $r_{d,j}$ becomes smaller, but the existing bound becomes smaller as $r_{d,j}$ becomes larger.
In this case, our bound describes the behavior of the second layer, and the existing bound describes that of the third and fourth layers.
See Appendix~\ref{app:experiment} for more details and results.

\section{Conclusion and discussion}\label{sec:conclusion}
In this paper, we proposed a new uniform complexity bound of neural networks using Koopman operators.
Our bound describes why networks with full-rank weight matrices generalize well and justifies the generalization property of networks with orthogonal matrices.
In addition, we provided an operator-theoretic approach to analyzing generalization property of neural networks.
There are several possible limitations.
First, our setting excludes non-smooth activation functions.
Generalizing our framework to other function spaces may help us understand these final nonlinear transformations and activation functions.
Moreover, although the factor $\Vert K_{\sigma_j}\Vert$ is bounded if we set certain activation functions, how this factor changes depending on the choice of the activation function has not been clarified yet.
Further investigation of this factor is required.
Furthermore, we assume a modification of the structure of the neural network for deriving a bound for non-injective matrices.
Thus, we still have room for improvement about this bound.
We simply decomposed the product of Koopman operators into every single Koopman operator.
For more refined analysis, they should be tied together to investigate the connection between layers.
Considering a function space on a manifold that contains the range of the transformation corresponding to each layer may also be effective for resolving this issue.
These challenging topics are left to future work.


\subsection*{Acknowledgement}
SS was partially supported by JST PRESTO JPMJPR2125 and JST CREST JPMJCR2015.
II was partially supported by JST CREST JPMJCR1913 and JST ACT-X JPMJAX2004.
AN was partially supported by JSPS KAKENHI 22H03650 and JST PRESTO JPMJPR1928.
TS was partially supported by JSPS KAKENHI 20H00576 and JST CREST.

\bibliography{koopman_nn}

\begin{thebibliography}{43}
\providecommand{\natexlab}[1]{#1}
\providecommand{\url}[1]{\texttt{#1}}
\expandafter\ifx\csname urlstyle\endcsname\relax
  \providecommand{\doi}[1]{doi: #1}\else
  \providecommand{\doi}{doi: \begingroup \urlstyle{rm}\Url}\fi

\bibitem[Anthony \& Bartlett(2009)Anthony and Bartlett]{anthony09}
Martin Anthony and Peter~L. Bartlett.
\newblock \emph{Neural network learning: Theoretical foundations}.
\newblock Cambridge University Press, 2009.

\bibitem[Arora et~al.(2018)Arora, Ge, Neyshabur, and Zhang]{arora18}
Sanjeev Arora, Rong Ge, Behnam Neyshabur, and Yi~Zhang.
\newblock Stronger generalization bounds for deep nets via a compression
  approach.
\newblock In \emph{Proceedings of the 35th International Conference on Machine
  Learning (ICML)}, 2018.

\bibitem[Bartlett \& Mendelson(2002)Bartlett and Mendelson]{bartlett02}
Peter~L. Bartlett and Shahar Mendelson.
\newblock Rademacher and {G}aussian complexities: Risk bounds and structural
  results.
\newblock \emph{Journal of Machine Learning Research}, 3:\penalty0 463--482,
  2002.

\bibitem[Bartlett et~al.(2017)Bartlett, Foster, and Telgarsky]{bartlett17}
Peter~L Bartlett, Dylan~J Foster, and Matus~J Telgarsky.
\newblock Spectrally-normalized margin bounds for neural networks.
\newblock In \emph{Proceedings of Advances in Neural Information Processing
  Systems 31 (NIPS)}, 2017.

\bibitem[Biswas et~al.(2022)Biswas, Kumar, Banerjee, and Pandey]{biswas22}
Koushik Biswas, Sandeep Kumar, Shilpak Banerjee, and Ashish~Kumar Pandey.
\newblock Smooth maximum unit: Smooth activation function for deep networks
  using smoothing maximum technique.
\newblock In \emph{Proceedings of 2022 IEEE/CVF Conference on Computer Vision
  and Pattern Recognition (CVPR)}, 2022.

\bibitem[Brunton et~al.(2022)Brunton, Budi\v{s}i\'{c}, Kaiser, and
  Kutz]{blunton22}
Steven~L. Brunton, Marko Budi\v{s}i\'{c}, Eurika Kaiser, and J.~Nathan Kutz.
\newblock Modern {K}oopman theory for dynamical systems.
\newblock \emph{SIAM Review}, 64\penalty0 (2):\penalty0 229--340, 2022.

\bibitem[Budi\v{s}i\'{c} et~al.(2012)Budi\v{s}i\'{c}, Mohr, and
  Mezi\'{c}]{mezic12}
Marko Budi\v{s}i\'{c}, Ryan Mohr, and Igor Mezi\'{c}.
\newblock Applied {K}oopmanism.
\newblock \emph{Chaos (Woodbury, N.Y.)}, 22:\penalty0 047510, 2012.

\bibitem[Das et~al.(2021)Das, Giannakis, and Slawinska]{das21}
Suddhasattwa Das, Dimitrios Giannakis, and Joanna Slawinska.
\newblock Reproducing kernel {H}ilbert space compactification of unitary
  evolution groups.
\newblock \emph{Applied and Computational Harmonic Analysis}, 54:\penalty0
  75--136, 2021.

\bibitem[Dogra \& Redman(2020)Dogra and Redman]{dogra20}
Akshunna~S. Dogra and William Redman.
\newblock Optimizing neural networks via {K}oopman operator theory.
\newblock In \emph{Proceedings of Advances in Neural Information Processing
  Systems 34 (NeurIPS)}, 2020.

\bibitem[Giannakis \& Das(2020)Giannakis and Das]{giannakis20}
Dimitrios Giannakis and Suddhasattwa Das.
\newblock Extraction and prediction of coherent patterns in incompressible
  flows through space-time {K}oopman analysis.
\newblock \emph{Physica D: Nonlinear Phenomena}, 402:\penalty0 132211, 2020.

\bibitem[Goldblum et~al.(2020)Goldblum, Geiping, Schwarzschild, Moeller, and
  Goldstein]{goldblum20}
Micah Goldblum, Jonas Geiping, Avi Schwarzschild, Michael Moeller, and Tom
  Goldstein.
\newblock Truth or backpropaganda? an empirical investigation of deep learning
  theory.
\newblock In \emph{Proceedings of the 8th International Conference on Learning
  Representations (ICLR)}, 2020.

\bibitem[Golowich et~al.(2018)Golowich, Rakhlin, and Shamir]{golowich18}
Noah Golowich, Alexander Rakhlin, and Ohad Shamir.
\newblock Size-independent sample complexity of neural networks.
\newblock In \emph{Proceedings of the 2018 Conference On Learning Theory
  (COLT)}, 2018.

\bibitem[Harvey et~al.(2017)Harvey, Liaw, and Mehrabian]{harvey17}
Nick Harvey, Christopher Liaw, and Abbas Mehrabian.
\newblock Nearly-tight {VC}-dimension bounds for piecewise linear neural
  networks.
\newblock In \emph{Proceedings of the 2017 Conference on Learning Theory
  (COLT)}, pp.\  1064--1068, 2017.

\bibitem[Hashimoto et~al.(2020)Hashimoto, Ishikawa, Ikeda, Matsuo, and
  Kawahara]{hashimoto20}
Yuka Hashimoto, Isao Ishikawa, Masahiro Ikeda, Yoichi Matsuo, and Yoshinobu
  Kawahara.
\newblock Krylov subspace method for nonlinear dynamical systems with random
  noise.
\newblock \emph{Journal of Machine Learning Research}, 21\penalty0
  (172):\penalty0 1--29, 2020.

\bibitem[He et~al.(2015)He, Zhang, Ren, and Sun]{kaiming15}
Kaiming He, Xiangyu Zhang, Shaoqing Ren, and Jian Sun.
\newblock Delving deep into rectifiers: Surpassing human-level performance on
  imagenet classification.
\newblock In \emph{Proceedings of 2015 IEEE International Conference on
  Computer Vision (ICCV)}, 2015.

\bibitem[Ikeda et~al.(2022{\natexlab{a}})Ikeda, Ishikawa, and
  Sawano]{ikeda22-2}
Masahiro Ikeda, Isao Ishikawa, and Yoshihiro Sawano.
\newblock Boundedness of composition operators on reproducing kernel {H}ilbert
  spaces with analytic positive definite functions.
\newblock \emph{Journal of Mathematical Analysis and Applications},
  511\penalty0 (1):\penalty0 126048, 2022{\natexlab{a}}.

\bibitem[Ikeda et~al.(2022{\natexlab{b}})Ikeda, Ishikawa, and
  Schlosser]{ikeda22}
Masahiro Ikeda, Isao Ishikawa, and Corbinian Schlosser.
\newblock Koopman and {P}erron-–{F}robenius operators on reproducing kernel
  {B}anach spaces.
\newblock \emph{Chaos: An Interdisciplinary Journal of Nonlinear Science},
  32\penalty0 (12):\penalty0 123143, 2022{\natexlab{b}}.

\bibitem[Ishikawa(2023)]{ishikawa23}
Isao Ishikawa.
\newblock Bounded composition operators on functional quasi-banach spaces and
  stability of dynamical systems.
\newblock \emph{Advances in Mathematics}, 424:\penalty0 109048, 2023.

\bibitem[Ishikawa et~al.(2018)Ishikawa, Fujii, Ikeda, Hashimoto, and
  Kawahara]{ishikawa18}
Isao Ishikawa, Keisuke Fujii, Masahiro Ikeda, Yuka Hashimoto, and Yoshinobu
  Kawahara.
\newblock Metric on nonlinear dynamical systems with {P}erron-{F}robenius
  operators.
\newblock In \emph{Proceedings of Advances in Neural Information Processing
  Systems 32 (NeurIPS)}, 2018.

\bibitem[Ju et~al.(2022)Ju, Li, and Zhang]{ju22}
Haotian Ju, Dongyue Li, and Hongyang~R Zhang.
\newblock Robust fine-tuning of deep neural networks with {H}essian-based
  generalization guarantees.
\newblock In \emph{Proceedings of the 39th International Conference on Machine
  Learning (ICML)}, 2022.

\bibitem[Kawahara(2016)]{kawahara16}
Yoshinobu Kawahara.
\newblock Dynamic mode decomposition with reproducing kernels for {K}oopman
  spectral analysis.
\newblock In \emph{Proceedings of Advances in Neural Information Processing
  Systems 30 (NIPS)}, 2016.

\bibitem[Kingma \& Ba(2015)Kingma and Ba]{kingma15}
Diederik~P. Kingma and Jimmy Ba.
\newblock {Adam}: A method for stochastic optimization.
\newblock In \emph{Proceedings of the 3rd International Conference on Learning
  Representations (ICLR)}, 2015.

\bibitem[Klus et~al.(2020)Klus, Schuster, and Muandet]{klus17}
Stefan Klus, Ingmar Schuster, and Krikamol Muandet.
\newblock Eigendecompositions of transfer operators in reproducing kernel
  {H}ilbert spaces.
\newblock \emph{Journal of Nonlinear Science}, 30:\penalty0 283--315, 2020.

\bibitem[Konishi \& Kawahara(2023)Konishi and Kawahara]{konishi23}
Takuya Konishi and Yoshinobu Kawahara.
\newblock Stable invariant models via {K}oopman spectra.
\newblock \emph{Neural Networks}, 165:\penalty0 393--405, 2023.

\bibitem[Koopman(1931)]{koopman31}
Bernard Koopman.
\newblock Hamiltonian systems and transformation in {H}ilbert space.
\newblock \emph{Proceedings of the National Academy of Sciences}, 17\penalty0
  (5):\penalty0 315--318, 1931.

\bibitem[Krizhevsky et~al.(2012)Krizhevsky, Sutskever, and
  Hinton]{krizhevsky12}
Alex Krizhevsky, Ilya Sutskever, and Geoffrey~E. Hinton.
\newblock Imagenet classification with deep convolutional neural networks.
\newblock In \emph{Proceedings of Advances in Neural Information Processing
  Systems 26 (NIPS)}, 2012.

\bibitem[Li et~al.(2021)Li, Jia, Wen, Liu, and Tao]{li21}
Shuai Li, Kui Jia, Yuxin Wen, Tongliang Liu, and Dacheng Tao.
\newblock Orthogonal deep neural networks.
\newblock \emph{IEEE Transactions on Pattern Analysis and Machine
  Intelligence}, 43\penalty0 (04):\penalty0 1352--1368, 2021.

\bibitem[Lusch et~al.(2017)Lusch, Nathan~Kutz, and Brunton]{lusch17}
Bethany Lusch, J.~Nathan~Kutz, and Steven~L. Brunton.
\newblock Deep learning for universal linear embeddings of nonlinear dynamics.
\newblock \emph{Nature Communications}, 9:\penalty0 4950, 2017.

\bibitem[Maduranga et~al.(2019)Maduranga, Helfrich, and Ye]{maduranga19}
Kehelwala D.~G. Maduranga, Kyle~E. Helfrich, and Qiang Ye.
\newblock Complex unitary recurrent neural networks using scaled {C}ayley
  transform.
\newblock In \emph{Proceedings of the 33rd AAAI Conference on Artificial
  Intelligence (AAAI)}, 2019.

\bibitem[Mohri et~al.(2018)Mohri, Rostamizadeh, and Talwalkar]{mohri18}
Mehryar Mohri, Afshin Rostamizadeh, and Ameet Talwalkar.
\newblock \emph{Foundations of Machine Learning}.
\newblock MIT press, 2018.

\bibitem[Neyshabur et~al.(2015)Neyshabur, Tomioka, and Srebro]{neyshabur15}
Behnam Neyshabur, Ryota Tomioka, and Nathan Srebro.
\newblock Norm-based capacity control in neural networks.
\newblock In \emph{Proceedings of the 2015 Conference on Learning Theory
  (COLT)}, 2015.

\bibitem[Neyshabur et~al.(2018)Neyshabur, Bhojanapalli, and
  Srebro]{neyshabur18}
Behnam Neyshabur, Srinadh Bhojanapalli, and Nathan Srebro.
\newblock A {PAC}-bayesian approach to spectrally-normalized margin bounds for
  neural networks.
\newblock In \emph{Proceedings of the 6th International Conference on Learning
  Representations (ICLR)}, 2018.

\bibitem[Neyshabur et~al.(2019)Neyshabur, Li, Bhojanapalli, LeCun, and
  Srebro]{neyshabur19}
Behnam Neyshabur, Zhiyuan Li, Srinadh Bhojanapalli, Yann LeCun, and Nathan
  Srebro.
\newblock The role of over-parametrization in generalization of neural
  networks.
\newblock In \emph{Proceedings of the 7th International Conference on Learning
  Representations (ICLR)}, 2019.

\bibitem[Novak et~al.(2018)Novak, Bahri, Abolafia, Pennington, and
  Sohl-Dickstein]{novak18}
Roman Novak, Yasaman Bahri, Daniel~A. Abolafia, Jeffrey Pennington, and Jascha
  Sohl-Dickstein.
\newblock Sensitivity and generalization in neural networks: an empirical
  study.
\newblock In \emph{Proceedings of the 6th International Conference on Learning
  Representations (ICLR)}, 2018.

\bibitem[Redman et~al.(2022)Redman, Fonoberova, Mohr, Kevrekidis, and
  Mezi\'{c}]{redman22}
William~T. Redman, Maria Fonoberova, Ryan Mohr, Yannis Kevrekidis, and Igor
  Mezi\'{c}.
\newblock An operator theoretic view on pruning deep neural networks.
\newblock In \emph{Proceedings of the 10th International Conference on Learning
  Representations (ICLR)}, 2022.

\bibitem[Rubin(2018)]{rubin18}
B.~A. Rubin.
\newblock A note on the blaschke-petkantschin formula, {R}iesz distributions,
  and {D}rury’s identity.
\newblock \emph{Fractional Calculus and Applied Analysis}, 21:\penalty0
  1641--1650, 2018.

\bibitem[Sawano(2018)]{sawano18}
Yoshihiro Sawano.
\newblock \emph{Theory of {B}esov Spaces}.
\newblock Springer Singapore, 2018.

\bibitem[Suzuki et~al.(2020)Suzuki, Abe, and Nishimura]{Suzuki20}
Taiji Suzuki, Hiroshi Abe, and Tomoaki Nishimura.
\newblock Compression based bound for non-compressed network: unified
  generalization error analysis of large compressible deep neural network.
\newblock In \emph{Proceedings of the 8th International Conference on Learning
  Representations (ICLR)}, 2020.

\bibitem[Tu(2011)]{tu08}
Loring~W. Tu.
\newblock \emph{An Introduction to Manifolds}.
\newblock Springer New York, second edition, 2011.

\bibitem[Wang et~al.(2020)Wang, Chen, Chakraborty, and Yu]{wang20}
Jiayun Wang, Yubei Chen, Rudrasis Chakraborty, and Stella~X. Yu.
\newblock Orthogonal convolutional neural networks.
\newblock In \emph{Proceedings of the IEEE/CVF Conference on Computer Vision
  and Pattern Recognition (CVPR)}, 2020.

\bibitem[Wei \& Ma(2019)Wei and Ma]{wei19}
Colin Wei and Tengyu Ma.
\newblock Data-dependent sample complexity of deep neural networks via
  {L}ipschitz augmentation.
\newblock In \emph{Proceedings of Advances in Neural Information Processing
  Systems 33 (NeurIPS)}, 2019.

\bibitem[Wei \& Ma(2020)Wei and Ma]{wei20}
Colin Wei and Tengyu Ma.
\newblock Improved sample complexities for deep neural networks and robust
  classification via an all-layer margin.
\newblock In \emph{Proceedings of the 8th International Conference on Learning
  Representations (ICLR)}, 2020.

\bibitem[{Weinan E} et~al.(2022){Weinan E}, Ma, and Wu]{weinan22}
{Weinan E}, Chao Ma, and Lei Wu.
\newblock The {B}arron space and the flow-induced function spaces for neural
  network models.
\newblock \emph{Constructive Approximation}, 55:\penalty0 369--406, 2022.

\end{thebibliography}
\bibliographystyle{iclr2024_conference}

\clearpage
\appendix
\setcounter{proposition}{0}
\renewcommand{\theproposition}{\Alph{proposition}}
\section*{Appendix}
\section{Proofs}\label{ap:proofs}
We provide the proofs of the theorems, propositions, and lemmas in the main text.
\begin{mythm}[Proposition~\ref{prop:activation_koopman}]
Let $p(\omega)=1/(1+\Vert \omega\Vert^2)^s$ for $\omega\in\r{d}$ $s\in\mathbb{N}$, and $s>d/2$.
If the activation function $\sigma$ has the following properties, then $K_{\sigma}:H_p(\r{d})\to H_p(\r{d})$ is bounded.
\begin{itemize}[nosep,leftmargin=*]
    \item $\sigma$ is $s$-times differentiable and its derivative $\partial^{\alpha}\sigma$ is bounded for any multi-index $\alpha\in\{(\alpha_1,\ldots,\alpha_d)\,\mid\,\sum_{j=1}^d\alpha_j\le s\}$.
    \item $\sigma$ is bi-Lipschitz, i.e., $\sigma$ is bijective and both $\sigma$ and $\sigma^{-1}$ are Lipschitz continuous.
\end{itemize}
\end{mythm}
\begin{proof}
For $h\in H(\r{d})$, we have $\Vert K_{\sigma}h\Vert_{H_p(\r{d})}^2=\sum_{\vert \alpha\vert\le s}\Vert\partial^{\alpha}(h\circ\sigma)\Vert_{L^2(\r{d})}^2$.
We denote $\sigma(x)=(\sigma_1(x),\ldots,\sigma_d(x))$ and $D_{\gamma}\sigma(x)=(\partial^{\gamma}\sigma_1(x),\ldots,\partial^{\gamma}\sigma_d(x))$ for $\gamma\in\mathbb{N}^d$.
By the Fa\`{a} di Bruno formula, we have
\red{\begin{align*}
    \partial^{\alpha}(h\circ\sigma)(x)=\sum_{\vert\beta\vert\le \vert\alpha\vert}\partial^{\beta}h(\sigma(x))\sum_{i=1}^{\vert\alpha\vert}\sum_{\gamma\in p(\alpha,\beta)}\alpha!\prod_{j=1}^i\frac{(\partial^{l_j}\sigma(x))^{k_j}}{k_j!(l_j!)^{\vert k_j\vert}},
\end{align*}
where $p(\alpha,\beta)=\{\gamma=(k_1,\ldots,k_s,l_1,\ldots,l_s)\,\mid\,0\le l_1\le\cdots\le l_s,\ \sum_{j=1}^sk_j=\alpha,\ \sum_{j=1}^s\vert k_j\vert l_j=\beta \}$.}
Thus, $\partial^{\alpha}(h\circ\sigma)(x)$ is written as the finite weighted sum of $\partial^{\beta}h(\sigma(x))\prod_{i=1}^m(D_{\gamma_i}\sigma(x))^{\delta_i}$ for some $m\le \vert\alpha\vert$ and $\beta,\gamma_i,\delta_i\in\mathbb{N}^d$, $\vert\beta\vert\le \vert\alpha\vert$, $\vert\gamma_i\vert\le \vert\alpha\vert$, $\vert\delta_i\vert\le \vert\alpha\vert$.
By the boundedness of the derivatives of $\sigma$, there exists ${C}_{\beta,\gamma,\delta}>0$ such that
\begin{align*}
\int_{\r{d}} \bigg\vert \partial^{\beta}h(\sigma(x))\prod_{i=1}^m(D_{\gamma_i}\sigma(x))^{\delta_i}\bigg\vert^2\mr{d}x
\le {C}_{\beta,\gamma,\delta}\int_{\r{d}} \vert \partial^{\beta}h(\sigma(x))\vert^2\mr{d}x.
\end{align*}
Moreover, by the Lipschitzness of $\sigma^{-1}$, there exists $\tilde{C}>0$ such that
\begin{align*}
\int_{\r{d}} \vert \partial^{\beta}h(\sigma(x))\vert^2\mr{d}x
\le\Vert \opn{det}(J\sigma^{-1})\Vert_{\infty}\int_{\r{d}} \vert \partial^{\beta}h(x)\vert^2\mr{d}x
\le \tilde{C}\int_{\r{d}} \vert \partial^{\beta}h(x)\vert^2\mr{d}x,
\end{align*}
\red{where $J\sigma^{-1}$ is the Jacobian of $\sigma^{-1}$,}
which shows the boundedness of $K_{\sigma}$.
\end{proof}

\begin{mythm}[Theorem~\ref{thm:nonsingular}]
The Rademacher complexity $\hat{R}_n(\mathbf{x},{F}_{\opn{inv}}(C,D))$ is bounded as 
\begin{align*}
\hat{R}_n(\mathbf{x},{F}_{\opn{inv}}(C,D))\le \frac{B\Vert g\Vert_{H_L}}{\sqrt{n}} \!\!\!\!\!\sup_{W_j\in\mcl{W}(C,D)}\bigg(\prod_{j=1}^L\frac{\Vert {p_{j}}/({p_{j-1}\circ W_j^*)}\Vert_{\infty}^{1/2}}{\vert \opn{det}(W_j)\vert^{1/2}}\bigg)\bigg(\prod_{j=1}^{L-1}\Vert K_{\sigma_j}\Vert\bigg).
\end{align*}
\end{mythm}

We use the following lemma to show Theorem~\ref{thm:nonsingular}.
\begin{lemma}\label{lem:Kw_norm} 
Assume $W_j:\r{d}\to\r{d}$ is invertible for $j=1,\ldots,L$.
Then, for $j=1,\ldots,L$, we have 
\begin{align*}
\Vert K_{W_j}\Vert\le \bigg(\bigg\Vert\frac{p_{j}}{p_{j-1}\circ W_j^*}\bigg\Vert_{\infty}\frac{1}{\vert\opn{det}(W_j)\vert}\bigg)^{1/2},\quad \Vert K_{b_j}\Vert=1.
\end{align*}
\end{lemma}

\begin{proof}
For $h\in H_{j}$, we have
$\widehat{(h\circ W_j)}(\omega)=\int_{\mathbb{R}^d}h(W_jx)\mr{e}^{-\mr{i}x\cdot \omega}\mr{d}x
={\hat{h}(W_j^{-*}\omega)}/{\vert\opn{det}(W_j)\vert}$.
Thus, the norm of the Koopman operator is bounded as
\begin{align*}
\Vert K_{W_j}h\Vert_{H_{j-1}}^2&=\int_{\mathbb{R}^d}\frac{\vert \hat{h}(W_j^{-*}\omega)\vert^2}{\vert\opn{det}(W_j)\vert^2p_{j-1}(\omega)}\mr{d}\omega
\le\Vert h\Vert_{H_j}^2\sup_{\omega\in\mathbb{R}^d}\bigg\vert\frac{p_{j}(\omega)}{p_{j-1}(W_j^*\omega)}\bigg\vert\frac{1}{\vert\opn{det}(W_j)\vert}.
\end{align*}
In addition, for $h\in H_{j}$, we have
$\widehat{(h\circ b_j)}(\omega)=\mr{e}^{-\mr{i}a_j\cdot\omega}\hat{h}(\omega)$.
Thus, we obtain $\Vert K_{b_j}h\Vert^2=\Vert h\Vert^2$.
\end{proof}

\begin{proof}[Proof of Theorem~\ref{thm:nonsingular}]
\red{Let $x_1,\ldots,x_n\in\mathbb{R}^{d_0}$ and $s_1,\ldots,s_n$ be i.i.d. Rademacher variables (random variables following the uniform distribution on $\{-1,1\}$.}
By the reproducing property of $H_0$ and the Cauchy--Schwartz inequality, we have
\begin{align}
&\frac{1}{n}\mr{E}\bigg[\sup_{f\in{F}_{\opn{inv}}(C,D)}\sum_{i=1}^ns_if(x_i)\bigg]
=\frac{1}{n}\mr{E}\bigg[\sup_{f\in{F}_{\opn{inv}}(C,D)}\sum_{i=1}^n\bracket{s_ik_{p_0}(\cdot,x_i),f}_{H_0}\bigg]\nn\\
&\qquad\le \frac{1}{n}\mr{E}\bigg[\sup_{f\in{F}_{\opn{inv}}(C,D)} \bigg(\sum_{i,j=1}^ns_is_jk_{p_0}(x_i,x_j)\bigg)^{1/2} \Vert f\Vert_{H_0}\bigg]\nn\\
&\qquad\le \frac{1}{n}\sup_{f\in{F}_{\opn{inv}}(C,D)}\Vert f\Vert_{H_0}\,\mr{E}^{\frac12}\bigg[\sum_{i,j=1}^ns_is_jk_{p_0}(x_i,x_j)\bigg]\nn\\
&\qquad\le\frac{1}{n}\sup_{f\in{F}_{\opn{inv}}(C,D)}\Vert f\Vert_{H_0}\bigg(\sum_{i=1}^nk_{p_0}(x_i,x_i)\bigg)^{1/2}\nn\\
&\qquad\le \frac{B}{\sqrt{n}}\sup_{W_j\in\mcl{W}(C,D)}\Vert K_{W_1}K_{b_1}K_{\sigma_1}\cdots K_{W_L}K_{b_L}g\Vert_{H_0}\nn\\
&\qquad\le\frac{B}{\sqrt{n}}\sup_{W_j\in\mcl{W}(C,D)}\bigg(\prod_{j=1}^L\Vert K_{W_j}\Vert\Vert K_{b_j}\Vert\Vert K_{\sigma_j}\Vert\bigg)\Vert g\Vert_{H_L},\label{eq:rademacher_proof}
\end{align}
where the third inequality is derived by the Jensen's inequality.
By Lemma~\ref{lem:Kw_norm}, we obtain the final result.
\end{proof}

\begin{mythm}[Lemma~\ref{lem:Wnorm}]
Let $p(\omega)=1/(1+\Vert\omega\Vert^2)^s$ for $s>d/2$ and $p_j=p$ for $j=0,\ldots,L$.
Then, we have
$\Vert{p}/({p\circ W_j^*})\Vert_{\infty}\le {\max\{1,{\Vert W_j\Vert^{2s}}\}}$.
\end{mythm}

\begin{proof}
By the definition of $p$, we have
\begin{align*}
\bigg\Vert\frac{p}{p\circ W_j^*}\bigg\Vert_{\infty}
&=\sup_{\omega\in\r{d}}\bigg\vert \frac{p(\omega)}{p(W_j^*\omega)}\bigg\vert
=\sup_{\omega\in\r{d}}\bigg\vert\bigg(\frac{1+\Vert W_j^*\omega\Vert^2}{1+\Vert \omega\Vert^2}\bigg)^s\bigg\vert
\le\max\{1,\Vert W_j\Vert^{2s}\}.
\end{align*}
\end{proof}

\begin{mythm}[Theorem~\ref{thm:injective}]
\red{Let $\mcl{W}_j(C,D)=\{W\in\r{d_{j-1}\times d_{j}}\,\mid\,d_j\ge d_{j-1},\ \Vert W\Vert\le C,\ \sqrt{\opn{det}({W^*W})}\ge D\}$ and $F_{\opn{inj}}(C,D)=\{f\in F\,\mid\, W_j\in\mcl{W}_j(C,D)\}$.}
The Rademacher complexity $\hat{R}_n(\mathbf{x},F_{\opn{inj}}(C,D))$ is bounded as 
\begin{align*}
\hat{R}_n(\mathbf{x},F_{\opn{inj}}(C,D)))\le \frac{B\Vert g\Vert_{H_L}}{\sqrt{n}} \!\!\!\!\!\sup_{W_j\in\mcl{W}_j(C,D)}\!\!\bigg(\prod_{j=1}^L\frac{\Vert {\red{p_{j-1}}}/({p_{j-1}\circ W_j^*})\Vert_{\mcl{R}(W_j),\infty}^{1/2}G_j}{{\opn{det}({W_j^*W_j)}}^{1/4}}\bigg)\bigg(\prod_{j=1}^{L-1}\Vert K_{\sigma_j}\Vert\bigg),
\end{align*}
\red{where $G_j=\Vert f_j|_{\mcl{R}(W_j)}\Vert_{H_{\red{p_{j-1}}}(\mcl{R}(W_j))}/\Vert f_j\Vert_{H_{j}}$ and $f_j=g\circ b_L \circ W_L\circ \sigma_{L-1}\circ b_{L-1}\circ W_{L-1}\circ\cdots \circ \sigma_j\circ b_j$.}
\end{mythm}
\begin{proof}
For $h\in H_j$, we have
\begin{align*}
\widehat{({h}\circ W_j)}(\omega)&=\int_{\mathbb{R}^{d_{j-1}}}h(W_jx)\mr{e}^{-\mr{i}x\cdot \omega}\mr{d}x
=\int_{\mcl{R}(W_j)}h(x)\mr{e}^{-\mr{i}x\cdot W_j^{-*}\omega}\mr{d}x\frac1{\vert\opn{det}(R_j)\vert}
=\frac{\hat{h}(W_j^{-*}\omega)}{\vert\opn{det}(R_j)\vert},
\end{align*}
where $W_j=Q_jR_j$ is the QR decomposition of $W_j$ and \red{$\mcl{R}(W_j)$ is the range of $W_j$}.
In addition, we regard $W_j:\r{d_{j-1}}\to\mcl{R}(W_j)$.
Since $\vert\opn{det}(R_j)\vert=(\opn{det}(R_j^*R_j))^{1/2}=(\opn{det}(W_j^*W_j))^{1/2}$, the norm of the Koopman operator is bounded as
\begin{align}
\Vert K_{W_j}{h}\Vert^2_{H_{j-1}}&=\int_{\mathbb{R}^{d_{j-1}}}\frac{\vert \hat{h}(W_j^{-*}\omega)\vert^2}{\opn{det}(W_j^*W_j) p_{j-1}(\omega)}\mr{d}\omega
=\int_{\mcl{R}(W_j)}\frac{\vert \hat{h}(\omega)\vert^2}{\opn{det}(W_j^*W_j)^{1/2} p_{j-1}(W_j^*\omega)}\mr{d}\omega\nn\\
&\le\Vert h|_{\mcl{R}(W_j)}\Vert_{H_{p_{\red{j-1}}}(\mcl{R}(W_j))}^2\sup_{\omega\in\mcl{R}(W_j)}\bigg\vert\frac{p_{\red{j-1}}(\omega)}{p_{j-1}(W_j^*\omega)}\bigg\vert\frac{1}{\opn{det}(W_j^*W_j)^{1/2}}.\label{eq:koopman_norm_injective}
\end{align}
Thus, we have
\begin{align*}
\Vert f\Vert_{H_0}&=\Vert K_{W_1}f_1\Vert_{H_0}\le \Vert f_1|_{\mcl{R}(W_1)}\Vert_{H_{\red{p_0}}(\mcl{R}(W_1))}\sup_{\omega\in\mcl{R}(W_1)}\bigg\vert\frac{\red{p_{0}}(\omega)}{p_{0}(W_1^*\omega)}\bigg\vert^{1/2}\frac{1}{\opn{det}(W_1^*W_1)^{1/4}}\\
&= G_1\Vert f_1\Vert_{H_1}\sup_{\omega\in\mcl{R}(W_1)}\bigg\vert\frac{\red{p_{0}}(\omega)}{p_{0}(W_1^*\omega)}\bigg\vert^{1/2}\frac{1}{\opn{det}(W_1^*W_1)^{1/4}}\\
&\le G_1\Vert K_{\sigma_1}\Vert_{H_1}\Vert K_{W_2}f_2\Vert_{H_1}\sup_{\omega\in\mcl{R}(W_1)}\bigg\vert\frac{\red{p_{0}}(\omega)}{p_{0}(W_1^*\omega)}\bigg\vert^{1/2}\frac{1}{\opn{det}(W_1^*W_1)^{1/4}}.
\end{align*}
Applying the inequality~\eqref{eq:koopman_norm_injective} iteratively, we obtain
\begin{align}
\Vert f\Vert_{H_0}\le \prod_{j=1}^L\bigg\Vert \frac{\red{p_{j-1}}}{p_{j-1}\circ W_j^*}\bigg\Vert_{\mcl{R}(W_j),\infty}^{1/2}\,\frac{G_j\Vert K_{\sigma_j}\Vert}{{\opn{det}({W_j^*W_j)}}^{1/4}}.\label{eq:f_norm}
\end{align}
Applying the inequality~\eqref{eq:f_norm} to $\Vert f\Vert_{H_0}$ in the inequality~\eqref{eq:rademacher_proof} completes the proof.
\end{proof}

\begin{mythm}[Lemma~\ref{lem:Wnorm_injective}]
Let $p_j(\omega)=1/(1+\Vert\omega\Vert^2)^{s_j}$ for $s_j>d_j/2$ and for $j=0,\ldots,L$.
Then, we have
$\Vert {\red{p_{j-1}}}/({p_{j-1}\circ W_j^*})\Vert_{\mcl{R}(W_j),\infty}\le {\max\{1,{\Vert W_j\Vert^{2s_{j-1}}}\}}$.
\end{mythm}
\begin{proof}
By the definition of $p_j$ and the assumption of $s_j\ge s_{j-1}$, we have
\begin{align*}
\bigg\Vert\frac{\red{p_{j-1}}}{p_{j-1}\circ W_j^*}\bigg\Vert_{\mcl{R}(W_j),\infty}
&=\sup_{\omega\in\mcl{R}(W_j)}\bigg\vert \frac{\red{p_{j-1}}(\omega)}{p_{j-1}(W_j^*\omega)}\bigg\vert\\
&\le \sup_{\omega\in\mcl{R}(W_j)}\bigg\vert\frac{(1+\Vert W_j^*\omega\Vert^2)^{s_{j-1}}}{(1+\Vert \omega\Vert^2)^{s_{j-1}}}\bigg\vert\\
&\le \max\{1,\Vert W_j\Vert^{2s_{j-1}}\}\sup_{\omega\in\mcl{R}(W_j)}\bigg\vert\bigg(\frac{1+\Vert \omega\Vert^2}{1+\Vert \omega\Vert^2}\bigg)^{s_{j-1}}\bigg\vert\\
&=\max\{1,\Vert W_j\Vert^{2s_{j-1}}\}.
\end{align*}
\end{proof}

\begin{mythm}[Lemma~\ref{lem:G_j}]
\red{Let $p_j(\omega)=1/(1+\Vert\omega\Vert^2)^{s_j}$ for $s_j>d_j/2$ and for $j=0,\ldots,L$.
Let $s_j\ge s_{j-1}$ and $\tilde{H}_j$ be the Sobolev space on $\mcl{R}(W_j)^{\perp}$ with $p(\omega)=1/(1+\Vert\omega\Vert^2)^{s_j-s_{j-1}}$.}
Then, $G_j\le G_{j,0}\Vert f_j|_{\mcl{R}(W_j)}\cdot f_j|_{\mcl{R}(W_j)^{\perp}}\Vert_{H_j}/\Vert f_j\Vert_{H_j}$, where 
$G_{j,0}=\Vert f_j\Vert_{\red{\tilde{H}_j}}^{-1}$.
\end{mythm}
\begin{remark}
\red{Since $W_j$ is injective, $\opn{dim}(\mcl{R}(W_j)^{\perp})=d_j-d_{j-1}$.
Thus, we have $s_j-s_{j-1}>\opn{dim}(\mcl{R}(W_j)^{\perp})/2$.}
\end{remark}
\begin{proof}
By the definition of $G_{j,0}$, \red{and since $\mcl{R}(W_j)$ and $\mcl{R}(W_j)^{\perp}$ are orthogonal}, we have
\begin{align*}
&\Vert f_j\Vert_{H_{\red{p_{j-1}}}(\mcl{R}(W_j))}^2=\int_{\mcl{R}(W_j)}\vert \hat{f_j}(\omega)\vert^2\frac{1}{\red{p_{j-1}}(\omega)}\mr{d}\omega\\
&=G_{j,0}^2\int_{\mcl{R}(W_j)}\vert \hat{f_j}(\omega_1)\vert^2\red{(1+\Vert\omega_1\Vert^2)^{s_{j-1}}}\mr{d}\omega_1\int_{\mcl{R}(W_j)^{\perp}}\vert \hat{f_j}(\omega_2)\vert^2\red{(1+\Vert\omega_2\Vert^2)^{s_j-s_{j-1}}}\mr{d}\omega_2\\
&=G_{j,0}^2\int_{\mcl{R}(W_j)}\int_{\mcl{R}(W_j)^{\perp}}\vert \hat{f_j}(\omega_1)\vert^2\vert \hat{f_j}(\omega_2)\vert^2\frac{1}{p_j(\omega_1+\omega_2)}\red{\frac{(1+\Vert\omega_1\Vert^2)^{s_{j-1}}(1+\Vert\omega_2\Vert^2)^{s_j-s_{j-1}}}{(1+\Vert\omega_1+\omega_2\Vert^2)^{s_j}}}\mr{d}\omega_2\mr{d}\omega_1\\
&\red{=G_{j,0}^2\int_{\mcl{R}(W_j)}\int_{\mcl{R}(W_j)^{\perp}}\vert \hat{f_j}(\omega_1)\vert^2\vert \hat{f_j}(\omega_2)\vert^2\frac{1}{p_j(\omega_1+\omega_2)}}\\
&\phantom{=G_{j,0}^2\int_{\mcl{R}(W_j)}\int_{\mcl{R}(W_j)^{\perp}}}\red{\cdot\frac{(1+\Vert\omega_1\Vert^2)^{s_{j-1}}(1+\Vert\omega_2\Vert^2)^{s_j-s_{j-1}}}{(1+\Vert\omega_1\Vert^2+\Vert\omega_2\Vert^2)^{s_{j-1}}(1+\Vert\omega_1\Vert^2+\Vert\omega_2\Vert^2)^{s_j-s_{j-1}}}\mr{d}\omega_2\mr{d}\omega_1}\\
&\le G_{j,0}^2\int_{\mcl{R}(W_j)}\int_{\mcl{R}(W_j)^{\perp}}\vert \hat{f_j}(\omega_1)\hat{f_j}(\omega_2)\vert^2\frac{1}{p_j(\omega_1+\omega_2)}\mr{d}\omega_2\mr{d}\omega_1\\
&=G_{j,0}^2\int_{\r{d_j}}\vert \widehat{\tilde{f_j}}(\omega)\vert^2\frac{1}{p_j(\omega)}\mr{d}\omega
=G_{j,0}^2\Vert \tilde{f_j}\Vert_{H_j}^2=G_{j,0}^2\Vert f_j\Vert_{H_j}^2\frac{\Vert \tilde{f_j}\Vert^2_{H_j}}{\Vert {f_j}\Vert^2_{H_j}},
\end{align*}
where $\tilde{h}(x)=h(x_1)h(x_2)$ for $x=x_1+x_2$, $x_1\in\mcl{R}(W_j)$, and $x_2\in\mcl{R}(W_j)^{\perp}$.
Note that since $\mcl{R}(W_j)$ and $\mcl{R}(W_j)^{\perp}$ are orthogonal, we have $\hat{h}(\omega_1)\hat{h}(\omega_2)=\widehat{\tilde{h}}(\omega)$.
\end{proof}

\begin{mythm}[Proposition~\ref{prop:weighted_koopman}]
Let $p_j(\omega)=1/(1+\Vert\omega\Vert^2)^{s_j}$ with $s_j\in\mathbb{N}$, $s_j\ge s_{j-1}$, and $s_j>d_j/2$.
Let $\psi_j$ be a function satisfying $\psi_j(x)=\psi_{j,1}(x_1)$ for some $\psi_{j,1}\in H_{p_{j-1}}(\opn{ker}(W_j))$, where $x=x_1+x_2$ for $x_1\in\opn{ker}(W_j)$ and $x_2\in\opn{ker}(W_j)^{\perp}$ \red{and $\opn{ker}(W_j)$ is the kernel of $W_j$}.
\red{Let $\mcl{W}_{\opn{r},j}(C,D)=\{W\in\r{d_{j-1}\times d_{j}}\,\mid\,\Vert W\Vert\le C,\ \vert\opn{det}(W_{\opn{r}})\vert\ge D\}$
and $F_{\mr{r}}(C,D)=\{f\in F\,\mid\,{W}_{j}\in{\mcl{W}}_{\mr{r},j}(C,D)\}$.}
\red{Moreover, let $G_j=\Vert f_j\Vert_{H_{j-1}(\opn{ker}(W_j)^{\perp})}/\Vert f_j\Vert_{H_j}$.}
Then, we have
\begin{align*}
\red{\hat{R}_n(\mathbf{x},{F}_{\mr{r}}(C,D)))\le \frac{B\Vert {g}\Vert_{H_L}}{\sqrt{n}}\!\!\!\!\!\!\sup_{{W}_j\in\mcl{W}_{\opn{r},j}(C,D)}\!\!\bigg(\prod_{j=1}^L\frac{\Vert \psi_{j,1}\Vert_{\tilde{H}_{j-1}}\!\!G_j\max\{1,\Vert W_j\Vert^{s_{j-1}}\}}{\vert\opn{det}(W_{{\opn{r},j}})\vert^{1/2}}\bigg)\bigg(\prod_{j=1}^{L-1}\Vert K_{\sigma_j}\Vert\bigg).}
\end{align*}
\end{mythm}
\begin{proof}
For $h\in H_j$, we have $\Vert\tilde{K}_{W_j}h\Vert^2_{H_{j-1}}=\sum_{\vert \alpha\vert\le s_{j-1}}\Vert\partial^{\alpha}(\psi_j\cdot h\circ W_j)\Vert_{L^2(\r{d_{j-1}})}^2$, where the directions of the derivatives are along the directions of $\opn{ker}(W_j)^{\perp}$ and $\opn{ker}(W_j)$.
We have
\begin{align}
\int_{\r{d_{j-1}}}\vert \partial^{\beta}\psi_j(x)\partial^{\gamma}(h\circ W_j)(x)\vert^2\mr{d}x
\le \Vert W_j\Vert^{2\vert\gamma\vert}\int_{\r{d_{j-1}}}\vert \partial^{\beta}\psi_j(x)(\partial^{\gamma}h)(W_jx)\vert^2\mr{d}x.\label{eq:weighted_1}
\end{align}
In addition, let $\phi$ be a function satisfying $\phi(x)=\phi_{1}(x_1)$ for some $\phi_{1}\in H_{p_{j-1}}(\opn{ker}(W_j))$, where $x=x_1+x_2$ for $x_1\in\opn{ker}(W_j)$ and $x_2\in\opn{ker}(W_j)^{\perp}$.
Let $u\in H_{p_{j-1}}(\r{d_{j-1}})$.
Then, we have
\begin{align}
\int_{\mathbb{R}^{d_{j-1}}}\vert \phi(x)u(W_jx)\vert^2\mr{d}x
&=\int_{\opn{ker}(W_j)}\int_{\opn{ker}(W_j)^{\perp}}\vert \phi_1(x_1)u(W_jx_2)\vert^2\mr{d}x_2\mr{d}{x_1}\nn\\
&=\int_{\opn{ker}(W_j)}\vert\phi_1(x_1)\vert^2\mr{d}x_1\int_{\opn{ker}(W_j)^{\perp}}\vert u(W_jx_2)\vert^2\mr{d}x_2\nn\\
&=\Vert \phi_1\Vert^2_{L^2(\opn{ker}(W_j))}\int_{\opn{ker}(W_j)^{\perp}}\vert u(W_jx_2)\vert^2\mr{d}x_2.\label{eq:weighted_2}
\end{align}
Combining Eqs.~\eqref{eq:weighted_1} and \eqref{eq:weighted_2}, we obtain
\begin{align*}
&\int_{\r{d_{j-1}}}\vert \partial^{\beta}\psi_j(x)(\partial^{\gamma}h\circ W_j)(x)\vert^2\mr{d}x
\le \Vert W_j\Vert^{2\vert\gamma\vert}\,\Vert \partial^{\beta}\psi_j\Vert^2_{L^2(\opn{ker}(W_j))}\int_{\opn{ker}(W_j)^{\perp}}\vert (\partial^{\gamma}h)(W_jx)\vert^2\mr{d}x\\
&\qquad\le \Vert W_j\Vert^{2\vert\gamma\vert}\,\Vert \partial^{\beta}\psi_j\Vert^2_{L^2(\opn{ker}(W_j))}\frac{1}{\vert\opn{det}(W_{\opn{r},j})\vert}\int_{\opn{ker}(W_j)^{\perp}}\vert \partial^{\gamma}h(x)\vert^2\mr{d}x.
\end{align*}
\red{As a result, we have
\begin{align*}
&\Vert\tilde{K}_{W_j}h\Vert^2_{H_{j-1}}=\sum_{\vert \alpha\vert\le s_{j-1}}c_{\alpha,s_{j-1},d_{j-1}}\Vert\partial^{\alpha}(\psi_j\cdot h\circ W_j)\Vert_{L^2(\r{d_{j-1}})}^2\\
&= \sum_{\vert \alpha\vert\le s_{j-1}}c_{\alpha,s_{j-1},d_{j-1}}\int_{\opn{ker}(W_j)}\int_{\opn{ker}(W_j)^{\perp}}\vert \partial^{\beta}\psi_{j,1}(x_1)\partial^{\alpha-\beta}(h\circ W_j)(x_2)\vert^2\mr{d}x_2\mr{d}x_1\\
&= \sum_{\vert \beta\vert\le s_{j-1}}\sum_{\vert \gamma\vert\le s_{j-1}-\vert\beta\vert}\!\!\!\!\!c_{\beta,s_{j-1},r_j}c_{\gamma,s_{j-1}-\vert\beta\vert,d_{j-1}-r_j}\int_{\opn{ker}(W_j)}\int_{\opn{ker}(W_j)^{\perp}}\!\!\!\!\!\!\!\!\!\!\!\!\!\!\!\vert\partial^{\beta}\psi_{j,1}(x_1)\partial^{\gamma}(h\circ W_j)(x_2)\vert^2\mr{d}x_2\mr{d}x_1\\
&\le \sum_{\vert \beta\vert\le s_{j-1}}\sum_{\vert \gamma\vert\le s_{j-1}-\vert\beta\vert}\!\!\!\!\!\!\!\!c_{\beta,s_{j-1},r_j}c_{\gamma,s_{j-1}-\vert\beta\vert,d_{j-1}-r_j}\,\Vert \partial^{\beta}\psi_{j,1}\Vert^2_{L^2(\opn{ker}(W_j))}\frac{\Vert W_j\Vert^{2\vert\gamma\vert}}{\vert\opn{det}(W_{\opn{r},j})\vert}\Vert \partial^{\gamma}h\Vert_{L^2(\opn{ker}(W_j)^{\perp})}^2\\
&\le \frac{\max\{1,\Vert W_j\Vert^{2s_{j-1}}\}}{\opn{det}(W_{\opn{r},j})} \sum_{\vert \beta\vert\le s_{j-1}}\!\!\!\!\!c_{\beta,s_{j-1},r_j}\Vert \partial^{\beta}\psi_{j,1}\Vert_{L^2(\opn{ker}(W_j))}^2\sum_{\vert \gamma\vert\le s_{j-1}}\!\!\!\!\!c_{\gamma,s_{j-1},d_{j-1}-r_j}\Vert \partial^{\gamma}h\Vert_{L^2(\opn{ker}(W_j)^{\perp})}^2\\
&\le \frac{\max\{1,\Vert W_j\Vert^{2s_{j-1}}\}}{\opn{det}(W_{\opn{r},j})}\Vert \psi_{j,1}\Vert_{H_{j-1}(\opn{ker}(W_j))}^2\Vert h\Vert_{H_{j-1}(\opn{ker}(W_j)^{\perp})}^2,
\end{align*}
where $\beta$ in the second line of the above formula is the multi index whose elements corresponding to $\opn{ker}(W_j)$ equal to those of $\alpha$ and other elements are zero.
In addition, $r_j=\opn{dim}(\opn{ker}(W_j))$ and $c_{\alpha,s,d}=(2\pi)^ds!/\alpha!/(s-\vert\alpha\vert)!$.}
Therefore, setting $h=f_j$, we have
\begin{align*}
\red{\Vert \tilde{K}_{W_j}f_j\Vert_{H_{j-1}}\le \Vert \psi_{j,1}\Vert_{H_{j-1}(\opn{ker}(W_j))}\frac{G_j}{\vert\opn{det}(W_{{\opn{r},j}})\vert^{1/2}}\max\{1,\Vert W_j\Vert^{s_{j-1}}\}\Vert{f_j}\Vert_{H_j},}
\end{align*}
which completes the proof of the proposition.
\end{proof}

\begin{mythm}[Proposition~\ref{prop:combine}]
Let $\tilde{\mathbf{x}}=(\tilde{x}_1,\ldots,\tilde{x}_n)\in(\r{d_l})^n$.
Let $v_n(\omega)=\sum_{i=1}^ns_i(\omega)k_{p_0}(\cdot,x_i)$, $\tilde{v}_n(\omega)=\sum_{i=1}^ns_i(\omega)k_{p_l}(\cdot,\tilde{x}_i)$, and $\gamma_n=\Vert v_n\Vert_{H_0}/\Vert \tilde{v}_n\Vert_{H_l}$.
Then, we have
\begin{align*}
\hat{R}_n(\mathbf{x},F_{l,\opn{comb}}(C,D))
&\le\sup_{\substack{W_j\in\mcl{W}_j(C,D)\\ (j=1,\ldots,l)}}\prod_{j=1}^l\bigg\Vert \frac{p_{j-1}}{p_{j-1}\circ W_j^*}\bigg\Vert_{\mcl{R}(W_j),\infty}^{1/2}\frac{G_j\Vert K_{\sigma_j}\Vert}{{\opn{det}({W_j^*W_j)}}^{1/4}}\\
&\cdot\bigg(\hat{R}_n(\tilde{\mathbf{x}},F_{l+1:L})+\frac{B}{\sqrt{n}}\inf_{h_1\in F_{l+1:L}}\!\!\!\mr{E}^{\frac12}\bigg[\sup_{h_2\in F_{l+1:L}}\!\bigg\Vert h_1-\frac{\Vert h_2\Vert_{H_l}\gamma_n}{\Vert \tilde{v}_n\Vert_{H_l}}\tilde{v}_n\bigg\Vert_{H_l}^2\bigg]\bigg).
\end{align*}
\end{mythm}
\begin{proof}
To simplify the notation, let
\begin{equation*}
    \beta_j=\bigg\Vert \frac{p_{j-1}}{p_{j-1}\circ W_j^*}\bigg\Vert_{\mcl{R}(W_j),\infty}^{1/2}\frac{G_j\Vert K_{\sigma_j}\Vert}{{\opn{det}({W_j^*W_j)}}^{1/4}}.
\end{equation*}
By the reproducing property of $H_0$ and the Cauchy--Shwartz inequality, we have
\begin{align*}
&\hat{R}_n(\mathbf{x},F_{l,\opn{comb}}(C,D))
=\frac{1}{n}\mr{E}\bigg[\sup_{f\in{F}_{l,\opn{comb}}(C,D)}\sum_{i=1}^ns_if(x_i)\bigg]
=\frac{1}{n}\mr{E}\bigg[\sup_{f\in{F}_{l,\opn{comb}}(C,D)}\bracket{v_n,f}_{H_0}\bigg]\\
&\le \frac{1}{n}\mr{E}\bigg[\sup_{f\in{F}_{l,\opn{comb}}(C,D)}\Vert v_n\Vert_{H_0} \Vert K_{W_1}K_{b_1}K_{\sigma_1}\cdots K_{W_{l}}K_{b_{l}}K_{\sigma_{l}} K_{W_{l+1}}K_{b_{l+1}}K_{\sigma_{l+1}}\cdots K_{W_L}K_{b_L}g\Vert_{H_0} \bigg]\\
&\le \frac{1}{n}\mr{E}\bigg[\sup_{f\in{F}_{l,\opn{comb}}(C,D)}\Vert v_n\Vert_{H_0} \prod_{j=1}^l\bigg\Vert \frac{p_{j-1}}{p_{j-1}\circ W_j^*}\bigg\Vert_{\mcl{R}(W_j),\infty}^{1/2}\frac{G_j\Vert K_{\sigma_j}\Vert}{{\opn{det}({W_j^*W_j)}}^{1/4}}\\
&\qquad\qquad\qquad\qquad\cdot\Vert K_{W_{l+1}}K_{b_{l+1}}K_{\sigma_{l+1}}\cdots K_{W_L}K_{b_L}g\Vert_{H_l} \bigg]\\
&\le \frac{1}{n}\mr{E}\bigg[\sup_{\substack{W_j\in\mcl{W}_j(C,D)\\(j=1,\ldots,l)}}\prod_{j=1}^l\beta_j\sup_{h_2\in F_{l+1:L}}\bracket{\tilde{v}_n,\frac{\Vert h_2\Vert_{H_l}\Vert v_n\Vert_{H_0}}{\Vert \tilde{v}_n\Vert_{H_l}^2}\tilde{v}_n}_{H_l} \bigg]\\
&= \frac{1}{n}\sup_{\substack{W_j\in\mcl{W}_j(C,D)\\(j=1,\ldots,l)}}\prod_{j=1}^l\beta_j
\mr{E}\bigg[\sup_{h_2\in F_{l+1:L}}\bigg(\bracket{\tilde{v}_n,h}_{H_l}+\bracket{\tilde{v}_n,\frac{\Vert h_2\Vert_{H_l}\gamma_n}{\Vert \tilde{v}_n\Vert_{H_l}}\tilde{v}_n-h}_{H_l} \bigg)\bigg]\\
&\le \frac{1}{n}\sup_{\substack{W_j\in\mcl{W}_j(C,D)\\(j=1,\ldots,l)}}\prod_{j=1}^L\beta_j\mr{E}\bigg[\sup_{h_1\in F_{l+1:L}}\bracket{\tilde{v}_n,h_1}_{H_l}+\sup_{h_2\in F_{l+1:L}}\bracket{\tilde{v}_n,\frac{\Vert h_2\Vert_{H_l}\gamma_n}{\Vert \tilde{v}_n\Vert_{H_l}}\tilde{v}_n-h}_{H_l}\bigg]\\
&\le \sup_{\substack{W_j\in\mcl{W}_j(C,D)\\(j=1,\ldots,l)}}\prod_{j=1}^L\beta_j \bigg(\hat{R}_n(\mathbf{x},F_{l+1:L})+\frac{1}{n}\mr{E}\bigg[\Vert \tilde{v}_n\Vert_{H_l}\sup_{h_2\in F_{l+1:L}} \bigg\Vert \frac{\Vert h_2\Vert_{H_l}\gamma_n}{\Vert \tilde{v}_n\Vert_{H_l}}\tilde{v}_n-h\bigg\Vert_{H_l}\bigg]\bigg)
\end{align*}
for any $h\in F_{l+1:L}$.
Moreover, again by the Cauchy--Schwartz inequality, we have
\begin{align*}
\mr{E}\bigg[\Vert \tilde{v}_n\Vert_{H_l}\sup_{h_2\in F_{l+1:L}} \bigg\Vert \frac{\Vert h_2\Vert_{H_l}\gamma_n}{\Vert \tilde{v}_n\Vert_{H_l}}\tilde{v}_n-h\bigg\Vert_{H_l}\bigg]
&\le \mr{E}^{\frac12}[\Vert v_n\Vert_{H_l}^2]\mr{E}^{\frac12}\bigg[\sup_{h_2\in F_{l+1:L}} \bigg\Vert \frac{\Vert h_2\Vert_{H_l}\gamma_n}{\Vert \tilde{v}_n\Vert_{H_l}}\tilde{v}_n-h\bigg\Vert^2_{H_l}\bigg]\\
&\le B\sqrt{n}\mr{E}^{\frac12}\bigg[\sup_{h_2\in F_{l+1:L}} \bigg\Vert \frac{\Vert h_2\Vert_{H_l}\gamma_n}{\Vert \tilde{v}_n\Vert_{H_l}}\tilde{v}_n-h\bigg\Vert^2_{H_l}\bigg],
\end{align*}
where the second inequality is derived in the same manner as the proof of Theorem~\ref{thm:nonsingular}.
Since $h\in F_{l+1:L}$ is arbitrary, we obtain the final result.
\end{proof}

\section{Details of Remark~\ref{rmk:k_sigma}}\label{ap:k_sigma}
\red{
To derive a bound $\Vert K_{\sigma}\Vert$, we bound $\sum_{\vert \alpha\vert\le s}c_{\alpha,s,d}\Vert \partial^{\alpha}(h\circ\sigma)\Vert^2$ by $\sum_{\vert \alpha\vert\le s}c_{\alpha,s,d}\Vert\partial^{\alpha}h\Vert^2$.
As the proof of the boundedness of $\Vert K_{\sigma}\Vert$, one strategy is using the Fa\`{a} di Bruno formula.}

\red{By the Fa\`{a} di Bruno formula, we have
\begin{align*}
    \partial^{\alpha}(h\circ\sigma)(x)=\sum_{\vert\beta\vert\le \vert\alpha\vert}\partial^{\beta}h(\sigma(x))\sum_{i=1}^{\vert\alpha\vert}\sum_{\gamma\in p(\alpha,\beta)}\alpha!\prod_{j=1}^i\frac{(\partial^{l_j}\sigma(x))^{k_j}}{k_j!(l_j!)^{\vert k_j\vert}},
\end{align*}
where $p(\alpha,\beta)=\{\gamma=(k_1,\ldots,k_s,l_1,\ldots,l_s)\,\mid\,0\le l_1\le\cdots\le l_s,\ \sum_{j=1}^sk_j=\alpha,\ \sum_{j=1}^s\vert k_j\vert l_j=\beta \}$.
If $\sigma$ is elementwise, $l_j$ and $k_j$ are chosen so that each of them has only one nonzero element, such as $(\vert l_j\vert,0,\ldots,0)$.
By counting the number of terms in the summation and calculating the coefficients of the terms, we can derive a bound of $\Vert K_{\sigma}\Vert$.
However, analytically representing the number of terms in the summation is a challenging task.
We admit that this strategy does not give us a tight bound.
There may be a more sophisticated approach to deriving a tight bound of $\Vert K_{\sigma}\Vert$.
However, the main goal of this paper is to investigate how the property of the weight matrices affects the generalization property.
Since $\Vert K_{\sigma}\Vert$ does not depend on the weight matrices, if we assume the structure of the network is given, the property of the weight matrices does not affect $\Vert K_{\sigma}\Vert$.
As we stated in Section~\ref{sec:conclusion}, investigating $\Vert K_{\sigma}\Vert$ and deriving a tighter bound is future work.}

\section{Details of Remark~\ref{rmk:G_j}}\label{ap:G_j}
We first show that $G_j$ is bounded by a constant that is independent of $f_j$.
Since $G_j$ depends on $\opn{ker}(W_j)$, we denote it by $G_j(\opn{ker}(W_j))$.
Let $\mcl{W}$ be a $k$-dimensional subspace of $\r{d}$ and $\{u_1,\ldots,u_k\}$ be an orthonormal basis of $\mcl{W}$.
We consider the average of $G_j(\mcl{W})$ on the Grassmann manifold $\mcl{G}_{d,k}$.
For this purpose, we fix an orthonormal basis $e_1,\ldots,e_d$ on $\mathbb{R}^d$ and denote by $\partial_if$ the derivative of $f$ in the direction of $e_i$.
In addition, we denote $\partial_U^{\alpha}=\prod_{j=1}^k(\sum_{i=1}^du_{j,i}\partial_i)^{\alpha_j}$, where $u_{j,i}=\bracket{u_j,e_i}$.
Let $s\in\mathbb{N}$.
Then, we have
\begin{align}
\Vert f\Vert_{H^s(\mcl{W})}^2
&=\sum_{\vert\alpha\vert\le s}c_{\alpha,s,d}\Vert \partial_U^{\alpha}f\Vert_{L^2(\mcl{W})}^2
\le \sum_{l=1}^s\sum_{\vert\alpha\vert=l}\sum_{\vert\beta\vert= l}c_{\alpha,s,d}D_{s,d,k}\Vert \partial^{\beta}f\Vert_{L^2(\mcl{W})}^2\nn\\
&=D_{s,d,k}\sum_{l=1}^s\sum_{\vert\alpha\vert=l}c_{\alpha,s,d}\sum_{\vert\beta\vert= l}\Vert \partial^{\beta}f\Vert_{L^2(\mcl{W})}^2\nn\\
&=D_{s,d,k}\sum_{\vert\alpha\vert\le s}c_{\alpha,s,d}\sum_{\vert\beta\vert\le s}\Vert \partial^{\beta}f\Vert_{L^2(\mcl{W})}^2\nn\\
&\le D_{s,d,k}(2\pi)^d(d+1)^s\sum_{\vert\beta\vert\le s}c_{\beta,s,d}\Vert \partial^{\beta}f\Vert_{L^2(\mcl{W})}^2.\label{eq:change_cor}
\end{align}
Here, we used the Cauchy--Schwartz inequality and derive the second inequality as follows:
\begin{align*}
\bigg\vert\prod_{j=1}^k\bigg(\sum_{i=1}^du_{j,i}\partial_i\bigg)^{\alpha_j}f\bigg\vert^2
&\le \bigg(\prod_{j=1}^k\bigg(\sum_{i=1}^du_{j,i}^2\bigg) ^{\alpha_j}\bigg)\bigg(\prod_{j=1}^k\bigg(\sum_{i=1}^d\mcl{D}_i^2\bigg)^{\alpha_j}f\bigg)\\
&=\prod_{j=1}^k\bigg(\sum_{i=1}^d\mcl{D}_i^2\bigg)^{\alpha_j}f
= \prod_{j=1}^k\sum_{\vert\beta\vert=\alpha_j}\binom{\alpha_j}{\beta}\mcl{D}_{\beta}^2f
\le D_{s,d,k}\sum_{\vert\beta\vert=\vert\alpha\vert}\vert\partial^{\beta}f\vert^2
\end{align*}
for some $D_{s,d,k}>0$ that depends on $s$, $d$, and $k$.
Here $\mcl{D}_i^2$ is the operator defined as $\mcl{D}_i^2f=\vert\partial_if\vert^2$ and $\mcl{D}_{\beta}^2f=\vert\partial_i^{\beta_i}f\vert^2$.
Assume $\{x\in\r{d}\,\mid\,\Vert x\Vert\le\epsilon\}$ is not contained in the support of $f$.
In this case, we have
\begin{align}
\sum_{\vert\beta\vert\le s}c_{\beta,s,d}\Vert \partial^{\beta}f\Vert_{L^2(\mcl{W})}^2
\le \epsilon^{k-d}\sum_{\vert\beta\vert\le s}c_{\beta,s,d}\int_{\mcl{W}}\vert\partial^{\beta}f(x)\vert^2\vert x\vert^{d-k}\mr{d}x\label{eq:shpere}
\end{align}
Integrating the both sides of \eqref{eq:shpere} and
by Theorem 2 by \citet{rubin18}, we have
\begin{align*}
\int_{\mcl{G}_{d,k}}\sum_{\vert\beta\vert\le s}c_{\beta,s,d}\Vert \partial^{\beta}f\Vert_{L^2(\mcl{W})}^2\mr{d}\mcl{W}
\le \epsilon^{k-d}\sum_{\vert\beta\vert\le s}c_{\beta,s,d}\frac{\sigma_d}{\sigma_k}\int_{\r{d}}\vert\partial^{\beta}f(x)\vert^2\mr{d}x,
\end{align*}
where $\sigma_d=2\pi^{d/2}/\Gamma(d/2)$ and $\Gamma$ is the Gamma function.
In addition, $\mr{d}\mcl{W}$ is the integration with respect to the $O(d)$-invariant probability
measure on $\mcl{G}_{d,k}$ and $O(d)$ is the orthogonal group in $\r{d}$.
Combining with Eq.~\eqref{eq:change_cor}, we obtain
\begin{align*}
\int_{\mcl{G}_{d,k}}\Vert f\Vert_{H^s(\mcl{W})}^2\mr{d}\mcl{W}
\le D_{s,d,k}(2\pi)^{d}(d+1)^s\epsilon^{k-d}\frac{\sigma_d}{\sigma_k}\Vert f\Vert_{H^s(\r{d})}^2.
\end{align*}
As a result, we obtain
\begin{align}
\int_{\mcl{G}_{d_j,d_{j-1}}}G_j(\mcl{W})^2\mr{d}\mcl{W}
&\le 
\int_{\mcl{G}_{d_j,d_{j-1}}} \frac{\Vert f\Vert_{H^{s_j}(\mcl{W})}^2}{\Vert f\Vert_{H_j}^2}\mr{d}\mcl{W}\nn\\
&\le D_{s_j,d_j,d_{j-1}}(d_j+1)^{s_j}(2\pi)^{d_j}\epsilon^{d_{j-1}-d_j}\frac{\sigma_{d_j}}{\sigma_{d_{j-1}}}.\label{eq:indep}
\end{align}
We admit that this inequality is not tight from the perspective of the dependence on $s_j$, $d_j$, and $d_{j-1}$.
However, surprisingly, the inequality~\eqref{eq:indep} shows 
that the factor $G_j$ is bounded by a constant that is independent of $f_j$ if $\{x\in\r{d}\,\mid\,\Vert x\Vert\le\epsilon\}$ is not contained in the support of $f_j$.
The assumption about $\{x\in\r{d}\,\mid\,\Vert x\Vert\le\epsilon\}$ can be satisfied if the input is transformed so that it does not take the value near $0$.

One of the reasons for the looseness of the above bound is that we upper bounded $\Vert f\Vert_{H^{s_{j-1}}(\mcl{W})}$ by $\Vert f\Vert_{H^{s_j}(\mcl{W})}$.
If $f_j$ can be controlled by the Gaussian, then the factor $G_j$ does not seriously affect the bound.
Indeed, let $\phi_c(x)=\mathrm{e}^{-\pi^2\Vert x\Vert^2/c}$.
In the case of $\vert \hat{f}_j(\omega_1+\omega_2)\vert\ge \vert\hat{f}_j(\omega_1)\vert\phi_c(\omega_2)$ for $\omega_1\in\mcl{R}(W_j)$ and $\omega_2\in\mcl{R}(W_j)^{\perp}$, we can evaluate $G_j$ as follows:
\begin{align}
\Vert &\phi_c\Vert_{H^s(\mathbb{R}^d)}^2=\int_{\omega\in\mathbb{R}^d}\vert\phi_c(\omega)\vert^2(1+\Vert\omega\Vert^2)^s\mathrm{d}\omega
=\int_{\omega\in\mathbb{R}^d}\mathrm{e}^{-2\pi^2\Vert\omega\Vert^2/c}(1+\Vert\omega\Vert^2)^s\mathrm{d}\omega\nn\\
&= \int_{0}^{\infty}\mathrm{e}^{-2\pi^2r^2/c}(1+r^2)^sr^{d-1}\mathrm{d}r\cdot 2\pi\prod_{i=1}^{d-2}\tilde{c}_i\nn\\
&= 2\pi\int_{0}^{\infty}\mathrm{e}^{-2\pi^2r^2/c}\sum_{i=0}^s\binom{s}{i}r^{2i+d-1}\mathrm{d}r\prod_{i=1}^{d-2}\tilde{c}_i\nn\\
&= 2\pi\sum_{i=0}^s\binom{s}{i}\int_{0}^{\infty}\mathrm{e}^{-t}t^{i+(d-1)/2}\bigg(\frac{c}{2\pi^2}\bigg)^{i+(d-1)/2}\frac{\sqrt{c}}{\pi\sqrt{8 t}}\mathrm{d}t\prod_{i=1}^{d-2}\tilde{c}_i\nn\\
&\sim c^{s+d/2}\pi^{-2s-d+1}2^{-s-d/2}\int_{0}^{\infty}\mathrm{e}^{-t}t^{s+d/2-1}\mathrm{d}t\prod_{i=1}^{d-2}\tilde{c}_i\nn\\
&=c^{s+d/2}\pi^{-2s-d+1}2^{-s-d/2}\Gamma(s+d/2)\prod_{i=1}^{d-2}\tilde{c}_i,\label{eq:gaussian_eval}
\end{align}
where $a\sim b$ means $a/b\to 1$ as $s\to\infty$ and $d\to\infty$. In addition, $\tilde{c}_i=\int_{0}^{\pi}\sin^{i}\theta\mr{d}\theta$.
Thus, we have
\begin{align*}
G_j^2&=\frac{\Vert f_j|_{\mathcal{R}(W_j)}\Vert_{H_{{p_{j-1}}}(\mathcal{R}(W_j))}^2}{\Vert f_j\Vert_{H_{j}}^2}
=\frac{\int_{\mcl{R}(W_j)}\vert\hat{f}_j(\omega_1)\vert^2(1+\Vert\omega_1\Vert^2)^{s_{j-1}}\mathrm{d}\omega_1}{\int_{\mathbb{R}^d}\vert\hat{f}_j(\omega_1+\omega_2)\vert^2(1+\Vert\omega_1\Vert^2+\Vert\omega_2\Vert^2)^{s_j}\mathrm{d}\omega}\\
&\le \frac{\int_{\mcl{R}(W_j)}\vert\hat{f}_j(\omega_1)\vert^2(1+\Vert\omega_1\Vert^2)^{s_{j-1}}\mathrm{d}\omega_1}{\int_{\mathbb{R}^d}\vert\hat{f}_j(\omega_1)\vert^2\phi_c(\omega_2)^2(1+\Vert\omega_1\Vert^2+\Vert\omega_2\Vert^2)^{s_j}\mathrm{d}\omega}\\
&\le\frac{\int_{\mcl{R}(W_j)}\vert\hat{f}_j(\omega_1)\vert^2(1+\Vert\omega_1\Vert^2)^{s_{j-1}}\mathrm{d}\omega_1}{\int_{\mathbb{R}^d}\vert\hat{f}_j(\omega_1)\vert^2\phi_c(\omega_2)^2(1+\Vert\omega_1\Vert^2)^{s_{j-1}}(1+\Vert\omega_2\Vert^2)^{\tilde{s}_j}\mathrm{d}\omega}\\
&=\frac{1}{\int_{\mcl{R}(W_j)^{\perp}}\phi_c(\omega_2)^2(1+\Vert\omega_2\Vert^2)^{\tilde{s}_j}\mathrm{d}\omega}\\
&\sim \bigg(c^{\tilde{s}_j+\tilde{d}_j/2}\pi^{-2\tilde{s}_j-\tilde{d}_j+1}2^{-\tilde{s}_j-\tilde{d}_j/2}\Gamma(\tilde{s}_j+{\tilde{d}_j}/{2})\prod_{i=1}^{\tilde{d}_j-2}\tilde{c}_i\bigg)^{-1},
\end{align*}
where $\tilde{s}_j=s_j-s_{j-1}$ and $\tilde{d}_j=d_j-d_{j-1}$.
Note that since $s_j\ge s_{j-1}$ and $d_j\ge d_{j-1}$, $G_j$ becomes small as $c$ becomes large and $s_j$ and $d_j$ becomes large.
The assumption $\vert \hat{f}_j(\omega_1+\omega_2)\vert\ge \vert\hat{f}_j(\omega_1)\vert\phi_c(\omega_2)$ means that $\hat{f}_j$ decays slower or equal to the speed of the Gaussian in the direction of $\omega_2$.
Even if $c$ is chosen small to satisfy the condition, the factor $\Gamma(\tilde{s}_j+{\tilde{d}_j}/{2})$ becomes sufficiently large if $s_j$ is sufficiently large.
As a result, the upper bound becomes sufficiently small if $s_j$ is sufficiently large.

\red{Moreover, the factor $G_j$ can alleviate the dependency of $\Vert K_{\sigma_j}\Vert$ on $d_j$ and $s_j$.
Even if the dependence of $\Vert K_{\sigma_j}\Vert$ on $d_j$ and $s_j$ is exponential, since the exponents appeared in the above evaluation are $-(d_j-d_{j-1})$ and $-(s_j-s_{j-1})$, we can expect that the dependency on $d_j$ and $s_j$ is reduced to the dependency on $d_{j-1}$ and $s_{j-1}$.}

\section{Details of Remark~\ref{rmk:bump_func}}\label{ap:bump_function}
\red{As an example of $\psi$, we can use a bump function ${\psi}(x)=1-g((\Vert x\Vert^2-a^2)/(b^2-a^2))$ on $\mathbb{R}^d$ for $0<a<b$, where $g(x)=f(x)/(f(x)-f(1-x))$, $f(x)=\mathrm{e}^{-1/x}$ for $x>0$, and $f(x)=0$ for $x\le 0$. 
In this case, the support of ${\psi}$ is $\{x\in\mathbb{R}^d\,\mid\,\Vert x\Vert\le b\}$ and $\psi(x)=1$ for $x\in \{x\in\mathbb{R}^d\,\mid\,\Vert x\Vert\le a\}$. 
If the output of each layer is bounded on $\{x\in\mathbb{R}^d\,\mid\,\Vert x\Vert\le a\}$, 
then we can obtain a modified network that is exactly the same on $\{x\in\mathbb{R}^d\,\mid\,\Vert x\Vert\le a\}$ with this bump function.
If $a$ and $b$ are large, then the support of $\psi$ becomes large, which makes the $L^2$-norm of $\psi$ large, and the Sobolev norm of $\psi$ also becomes large.
If $a-b$ is small, then $\vert \psi(x)-\psi(y)\vert/\Vert x-y\Vert$ for $\Vert x\Vert^2=a$ and $y=(b/a)x$ becomes large.
Thus, the $L^2$-norms of the derivatives of $\psi$ are expected to be large, and the Sobolev norm of $\psi$ also becomes large if $a-b$ is small.}

\section{Details of Remark~\ref{rmk:g_grow}}\label{ap:g_grow}
The factor $\Vert \tilde{g}\Vert_{\tilde{H}_L}$ grows as $\Omega$ becomes large,
\red{where $\Omega$ is the region such that $\tilde{f}(x)=f(x)$ on $x\in\Omega$ for the network $f$ and the modified network $\tilde{f}$. }
Indeed, if $p_L(\omega)=1/(1+\Vert\omega\Vert^2)^{s_L}$ for $s_L\in\mathbb{N}$, then we have
\begin{align*}
\Vert \tilde{g}\Vert_{\tilde{H}_L}^2&=
\sum_{\vert \alpha\vert\le s_L}\red{c_{\alpha,s_L,\delta_{L-1}+d_L}}\Vert\partial^{\alpha}(\psi\cdot g)\Vert_{L^2(\r{\delta_L})}^2\\
&=\red{\sum_{\vert \alpha\vert\le s_L}c_{\alpha,s_L,\delta_{L-1}+d_L}\Vert\partial^{\beta}\psi\partial^{\alpha-\beta}g\Vert_{L^2(\r{\delta_{L-1}+d_L})}^2}\\
&=\red{\sum_{\vert \alpha\vert\le s_L}c_{\alpha,s_L,\delta_{L-1}+d_L}\Vert\partial^{\beta}\psi\Vert_{L^2(\r{\delta_{L-1}})}^2\Vert\partial^{\alpha-\beta}g\Vert_{L^2(\r{d_L})}^2,}
\end{align*}
\red{where $\beta$ is the multi index whose elements corresponding to $\delta_{L-1}$ equals to those of $\alpha$ and the other elements are zero.}
The factor $\Vert \psi\Vert^2_{L^2(\r{\delta_{L-1}})}$ becomes large as the volume of $\Omega$ gets large.
Thus, under the condition that $\Vert \partial^{\alpha}\psi\Vert_{L^2(\r{\delta_{L-1}})}$ does not change, $\Vert \tilde{g}\Vert_{\tilde{H}_L}$ becomes large as $\Omega$ gets large. 
\red{Indeed, assume $\delta_L=1$, $\psi(x)=1$ for $x\in\Omega$, and $\Omega$ is an interval (e.g., $\psi$ is the bump function defined in Appendix~\ref{ap:bump_function}).
Then we can create a new function $\tilde{\psi}$ that satisfies $\Vert \tilde{\psi}^{(i)}\Vert_{L^2(\mathbb{R})}=\Vert \psi^{(i)}\Vert_{L^2(\mathbb{R})}$ for any $i=1,2,\ldots$, from $\psi$ as follows.
Here, $\psi^{(i)}$ is the $i$th derivative of $\psi$.
For simplicity, we consider the case where $\Omega=[-a,a]$ for some $a>0$.
For $c>0$, we set $\tilde{\psi}(x)=1$ for $x\in [0,c]$, $\tilde{\psi}(x)=\psi(x-c)$ for $x\in (c,\infty)$, $\tilde{\psi}(x)=1$ for $x\in [-c,0]$, $\tilde{\psi}(x)=\psi(x+c)$ for $x\in (-\infty,0)$.}

\section{Details of Remark~\ref{rmk:psi_norm}}\label{ap:psi_norm}
\red{In Proposition~\ref{prop:weighted_koopman}, by combining with the factor $G_j$, the norm of $\psi_j$ can be canceled as follows.
Let $p_j(\omega)=1/(1+\Vert\omega\Vert^2)^{s_j}$ and $s_j=2s_{j-1}$.
We have
\begin{align*}
&\Vert\psi_j\Vert^2_{H_{j-1}(\mathrm{ker}(W_j))}\Vert f_j\Vert^2_{H_{j-1}(\mathrm{ker}(W_j)^{\perp})}\\
&\qquad=\int_{\mathrm{ker}(W_j)}\vert \hat{\psi_j}(\omega_1)\vert^2(1+\Vert\omega_1\Vert^2)^{s_{j-1}}\mathrm{d}\omega_1\int_{\mathrm{ker}(W_j)^{\perp}}\vert \hat{f_j}(\omega_2)\vert^2 (1+\Vert\omega_2\Vert^2)^{s_{j-1}}\mathrm{d}\omega_2\\
&\qquad\le\int_{\mathrm{ker}(W_j)}\int_{\mathrm{ker}(W_j)^{\perp}}\vert \hat{\psi_j}(\omega_1)\hat{f_j}(\omega_2)\vert^2(1+\Vert\omega_1+\omega_2\Vert^2)^{s_{j}}\mathrm{d}\omega_1\mathrm{d}\omega_2\\
&\qquad=\Vert\psi_{j,1}f_j|_{\mathrm{ker}(W_j)^{\perp}}\Vert_{H_j(\r{d_{j-1}})}^2.
\end{align*}
Thus, we obtain
\begin{align*}
\Vert\psi_j\Vert_{H_{j-1}(\mathrm{ker}(W_j))}G_j&=\frac{\Vert\psi_j\Vert_{H_{j-1}(\mathrm{ker}(W_j))}\Vert f_j\Vert_{H_{j-1}(\mathrm{ker}(W_j)^{\perp})}}{\Vert f_j\Vert_{H_j}}\\
&\le \frac{\Vert\psi_{j,1}f_j|_{\mathrm{ker}(W_j)^{\perp}}\Vert_{H_j(\mathbb{R}^{d_{j-1}})}}{\Vert f_j\Vert_{H_j}}.
\end{align*}
If $\psi_{j,1}f_j|_{\mathrm{ker}(W_j)^{\perp}}=f_j$, e.g. $\psi_{j,1}$ and $f_j$ are the Gaussian, and $d_{j-1}=d_j$, then the factor $\Vert\psi_j\Vert_{H_{j-1}(\mathrm{ker}(W_j))}G_j$ is bounded by $1$.}

\section{Details of Remark~\ref{rmk:combine}}\label{ap:combine}
We can combine our Koopman-based approach with the existing ``peeling'' approach.
For $1\le l\le L$, let $F^{\opn{ReLU}}_l(C_1)$ be the set of $l$-layer ReLU networks where the Frobenius norms of $W_1,\ldots,W_l$ are bounded by $C_1$, considered by~\citet{neyshabur15}.
Let $\tilde{F}_{1:l}$ be the set of functions defined in the same manner as Eq.~\eqref{eq:nn_middle}, except for replacing $\sigma_l$ with $g\in H_l$.
In addition, let $\tilde{F}_{1:l,\opn{inj}}(C_2,D)=\{f\in \tilde{F}_{1:l}\,\mid\, W_j\in\mcl{W}_j(C_2,D)\}$.
We combine $F^{\opn{ReLU}}_{L-l}(C_1)$ and $\tilde{F}_{1:l,\opn{inj}}(C_2,D)$, and we define
$F^{\opn{ReLU},L}_{1:l,\opn{inj}}(C_1,C_2,D)=\{h\circ f\,\mid\, h\in F^{\opn{ReLU}}_{L-l}(C_1),\ f\in \tilde{F}_{1:l,\opn{inj}}(C_2,D)\}$.
Then, applying Theorem 1 by~\citet{neyshabur15}, we can bound the complexity of $L$-layer networks using that of $(L-1)$-layer networks and the Frobenius norm of $W_L$.
For example, by Eq.~(8) by \citet{golowich18}), we have
\begin{align*}
\hat{R}_n(\mathbf{x},F^{\opn{ReLU,L}}_{1:l,\opn{inj}}(C_1,C_2,D))
&\le \Vert W_L\Vert_{2,2}\hat{R}_n\big(\mathbf{x},\sigma\big(F^{\opn{ReLU,L-1}}_{1:l,\opn{inj}}(C_1,C_2,D)\big)\big)\\
&\le 2\Vert W_L\Vert_{2,2}\hat{R}_n(\mathbf{x},F^{\opn{ReLU,L-1}}_{1:l,\opn{inj}}(C_1,C_2,D)),
\end{align*}
where $\hat{R}_n(\mathbf{x},\mcl{F})$ for a vector-valued function class $\mcl{F}$ is defined as $\mr{E}[\sup_{f\in\mcl{F}}1/n\Vert\sum_{i=1}^ns_if(x_i)\Vert]$ and $\sigma$ is the ReLU.
As a result, we obtain
\begin{align*}
\hat{R}_n(\mathbf{x},F^{\opn{ReLU,L}}_{1:l,\opn{inj}}(C_1,C_2,D))
&\le 2\Vert W_L\Vert_{2,2}\hat{R}_n(\mathbf{x},F^{\opn{ReLU,L-1}}_{1:l,\opn{inj}}(C_1,C_2,D))\\
&\le 
2^{L-l}\prod_{j=l+1}^{L}\Vert W_j\Vert_{2,2}\hat{R}_n(\mathbf{x},\tilde{2,2}_{1:l,\opn{inj}}(C_2,D))\\
&\le 2^{L-l}\Vert g\Vert_{H_l}\bigg(\prod_{j=l+1}^{L}\Vert W_j\Vert_{2,2}\bigg)\,\bigg(\prod_{j=1}^{l}\frac{G_j\Vert K_{\sigma_j}\Vert\Vert W_j\Vert^{s_j}}{\opn{det}(W_j^*W_j)^{1/4}}\bigg).
\end{align*}
We can also apply other peeling approaches in the same manner as the above case.

\section{Examples of concrete Koopman-based bounds}
\red{We show examples of our Koopman-based bounds.}
\begin{example}\label{ex:bound1}
\red{Let $g(x)=\mathrm{e}^{-c\Vert x\Vert^2}$. 
Let $p_j(\omega)=1/(1+\Vert\omega\Vert^2)^{s_j}$ for $s_j>d_j/2$.
We consider a shallow network $f(x)=g(Wx+b)$.
Assume $d_1\ge d_0$ and $W$ is full-rank.
The final nonlinear transformation $g$ in $f$ maps the high dimensional vector on $\mathbb{R}^{d_1}$ to a scalar value. 
In this case, $f_1(x)=g(x+b)$ is also the Gaussian.
We have
$$\hat{R}_n(\mathbf{x},F_{\mathrm{inj}}(C,D))\le \frac{B}{\sqrt{n}}G_{1}\Vert g\Vert_{H_1} \frac{\max\{1,\Vert W\Vert\}^{s_0}}{\mathrm{det}(W^*W)^{1/4}}.$$
Since $\Vert g\Vert_{H_1}=\Vert f_1\Vert_{H_1}$, by Eq.~\eqref{eq:gaussian_eval}, we have
\begin{align*}
G_1^2\Vert g\Vert_{H_1}^2=\frac{\Vert f_1|_{\mathcal{R}(W)}\Vert_{H_{{p_0}(\mathcal{R}(W))}}\Vert g\Vert_{H_1}}{\Vert f_1\Vert_{H_{1}}}
\sim c^{s_0-d_0/2}\pi^{-2s_0+1}2^{-s_0-d_0/2}\Gamma(s_0+d_0/2)\prod_{i=1}^{d_0-2}\tilde{c}_i.
\end{align*}
As a result, we have
\begin{align*}
&\hat{R}_n(\mathbf{x},F_{\mathrm{inj}}(C,D)) \\
&\lesssim\frac{B}{\sqrt{n}}c^{s_0/2-d_0/4}\pi^{-s_0+1/2}2^{-s_0/2-d_0/4}\bigg(\prod_{i=1}^{d_0-2}\tilde{c}_i\bigg)^{1/2}\Gamma(s_0+d_0/2)^{1/2}\frac{\max\{1,\Vert W\Vert\}^{s_0}}{\mathrm{det}(W^*W)^{1/4}}.
\end{align*}
Note that $\mathrm{dim}(\mathcal{R}(W))=d_0$ is the dimension of the input and $s_0$ is chosen as $s_0>d_0/2$.
They are independent of the structure of the network.}
\end{example}

\begin{example}
\red{Let $\sigma(x)=(\mathrm{e}^{-c_1\Vert x\Vert^2},\ldots,\mathrm{e}^{-c_{d_1}\Vert x\Vert^2})$. 
Let $p_j(\omega)=1/(1+\Vert\omega\Vert^2)^{s_j}$ for $s_j>d_j/2$.
We consider a shallow network $f(x)=W_2\sigma(W_1x+b)$.
Assume $d_1\ge d_0$, $d_2=1$ and $W$ is full-rank.
Using the ``peeling'' approach and Example~\ref{ex:bound1}, we obtain
\begin{align*}
&\hat{R}_n(\mathbf{x},F_{\mathrm{1:1,inj}}^2(C_1,C_2,D))\le \Vert W_2\Vert_{2,2}\hat{R}_n(\mathbf{x},F_{\mathrm{inj}}^1(C_2,D))\\
&\lesssim \Vert W_2\Vert_{2,2}\sum_{i=1}^{d_1}\frac{B}{\sqrt{n}}c_i^{s_0/2-d_0/4}\pi^{-s_0+1/2}2^{-s_0/2-d_0/4}\bigg(\prod_{i=1}^{d_0-2}\tilde{c}_i\bigg)^{1/2}\!\!\!\Gamma(s_0+d_0/2)^{1/2}\frac{\max\{1,\Vert W_1\Vert\}^{s_0}}{\mathrm{det}(W_1^*W_1)^{1/4}}.
\end{align*}
Here, we used the inequality
\begin{align*}
\hat{R}_n(\mathbf{x},\mcl{F})
&=\frac{1}{n}\mr{E}\bigg[\sup_{f\in\mcl{F}}\bigg\Vert\sum_{i=1}^ns_if(x_i)\bigg\Vert\bigg]
=\frac{1}{n}\mr{E}\bigg[\sup_{f\in\mcl{F}}\sqrt{\sum_{j=1}^d\bigg(\sum_{i=1}^ns_i(f(x_i))_j\bigg)^2}\bigg]\\
&\le \frac{1}{n}\mr{E}\bigg[\sup_{f\in\mcl{F}}\sum_{j=1}^d\sqrt{\bigg(\sum_{i=1}^ns_i(f(x_i))_j\bigg)^2}\bigg]
=\frac{1}{n}\mr{E}\bigg[\sup_{f\in\mcl{F}}\sum_{j=1}^d\bigg\vert\sum_{i=1}^ns_i(f(x_i))_j\bigg\vert\bigg]\\
&\le\frac{1}{n}\sum_{j=1}^d\mr{E}\bigg[\sup_{f\in\mcl{F}}\bigg\vert\sum_{i=1}^ns_i(f(x_i))_j\bigg\vert\bigg],
\end{align*}
where $(f(x_i))_j$ is the $j$th element in the vector $f(x_i)$.}
In this case, the bound depends on $d_1$ linearly.
\end{example}

\section{Injectivity of $\tilde{W}_j$}\label{ap:injective}
\red{The operator $\tilde{W}_j$ defined in Subsection~\ref {subsec:injective_based} is injective.
Indeed, assume $(W_1x,P_1x)=(W_1y,P_1y)$ for $x,y\in\mathbb{R}^{d_0}$.
Then, we have $x-y\in ker(W_1)$.
On the other hand, we have $P_1(x-y)=0$.
Since $x-y\in ker(W_1)$, we have $x-y=P_1(x-y)=0$.
The case of $j\ge 2$ is the same.}

\section{Experimental details and additional experimental results}\label{app:experiment}
We explain details of experiments in Section~\ref{sec:numerical_results} and show additional experimental results.
All the experiments were executed with Python 3.9 and TensorFlow 2.6.
\subsection{Validity of the bound (Synthetic data)}\label{app:ex_detail1}
We constructed a network $f_{\theta}(x)=g(W_2\sigma(W_1x+b_1)+b_2)$, where $W_1\in\r{3\times 3}$, $W_2\in\r{6\times 3}$, $b_1\in\r{3}$, $b_2\in\r{6}$, $\theta=(W_1,b_1,W_2,b_2)$, $\sigma(x)=((1+\alpha)x+(1-\alpha)x\,\mr{erf}(\mu(1-\alpha)x))/2$, and $g(x)=\mr{e}^{-\Vert x\Vert^2}$.
Here, $\mr{erf}$ is the Gaussian error function, and $\sigma$ is a smooth version of Leaky ReLU proposed by~\citet{biswas22}.
We set $\alpha=\mu=0.5$.
For training the network, we used $n=1000$ samples $x_i\ (i=1,\ldots,1000)$ drawn independently from the normal distribution with mean 0 and standard deviation 1.
The weight matrices are initialized by Kaiming Initialization~\citep{kaiming15}, and we used the SGD for the optimizer.
In addition, we set the error function as $l_{\theta}(x,y)=\vert f_{\theta}(x)-y\vert^2$, and added the regularization term $0.01(\prod_{j=1}^2\opn{det}(W_j^*W_j)^{-1/2}+10\prod_{j=1}^2\Vert W_j\Vert)$.
We added this regularization term since, according to our bound, both the determinant and the operator norm of $W_j$ should be small for achieving a small generalization error.
The generalization error here means $\vert\mathrm{E}[l_{{\theta}}(x,t(x))]-1/n\sum_{i=1}^nl_{{\theta}}(x_i,t(x_i))\vert$, which is compared to our bound $O(\prod_{j=1}^L\Vert W_j\Vert^{s_j}/(\mathrm{det}(W_j^*W_j)^{1/4}))$ in Figure~\ref{fig:mainfig} (a).
Here, we set $s_j=(d_j+0.1)/2$.

\subsection{Validity of the bound (MNIST)}\label{app:mnist}
We constructed a network $f_{\theta}(x)=g(W_4\sigma(W_3\sigma(W_2\sigma(W_1x+b_1)+b_2)+b_3)+b_4)$ with dense layers, where $W_1\in\r{1024\times 784}$, $W_2\in\r{2048\times 1024}$, $W_3\in\r{2048\times 2048}$, $W_4\in\r{10\times 2048}$, $b_1\in\r{1024}$, $b_2\in\r{2048}$, $b_3\in\r{2048}$, $b_4\in\r{10}$, $\theta=(W_1,b_1,W_2,b_2,W_3,b_3,W_4,b_4)$, $\sigma$ is the same function as Section~\ref{app:ex_detail1}, and $g$ is the softmax.
See Remark~\ref{rmk:final_trans} for the validity of our bound for the case where $g$ is the softmax.
For training the network, we used only $n=1000$ samples to create a situation where the model is hard to generalize.
We consider the regularization term $\lambda_1\Vert W_j\Vert +\lambda_2/\opn{det}(I+W_j^*W_j)$, where $\lambda_1=\lambda_2=0.01$ to make both the norm and the determinant of $W_j$ small (thus, makes our bound small).
Based on the observation in the last part of Subsection~\ref{subsec:combine}, we set the regularization term for only $j=1,2$.
The weight matrices are initialized by the orthogonal initialization for $j=1,2$ and by the samples from a truncated normal distribution for $j=3,4$, and we used Adam~\citep{kingma15} for the optimizer.
In addition, we set the error function as the categorical cross-entropy loss.

\subsubsection{Transformation of signals by lower layers}\label{subsec:transformation}
We also investigated the transformation by lower layers, as we stated in the last part of Subsection~\ref{subsec:combine}.
We computed $\vert \cos(\theta)\vert$,
where $\theta$ is the maximum value of the angles between the output of the second layer and the directions of singular vectors of $W_3$ associated with the singular values that are larger than $0.1$.
The results are illustrated in Figure~\ref{fig:angle}.
We can see that with the regularization based on our bound, as the test accuracy grows, the value $\vert \cos(\theta)\vert$ also grows.
This result means that the signals turn to the directions of the singular vectors of the subsequent weight matrix associated with large singular values.
That makes the extraction of information from the signals in higher layers easier.
On the other hand, without the regularization, neither the test accuracy nor the value $\vert \cos(\theta)\vert$ do not become large after a sufficiently long learning process (see also Figure~\ref{fig:mainfig} (b)).
The results in Figures~\ref{fig:mainfig} (b) and \ref{fig:angle} are obtained by three independent runs.

\begin{figure}[t]
    \centering
    \includegraphics[scale=0.3]{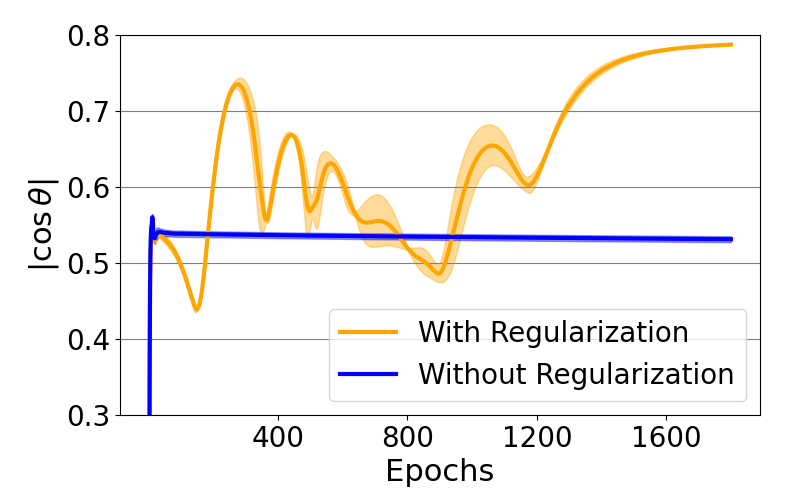}
    \caption{Behavior of the value $\vert \cos(\theta)\vert$.
Here, $\theta$ is the maximum value of the angles between the output of the second layer and the directions of singular vectors of $W_3$ associated with the singular values that are larger than $0.1$.}
    \label{fig:angle}
\end{figure}

\begin{figure}[t]
    \centering
    \includegraphics[scale=0.3]{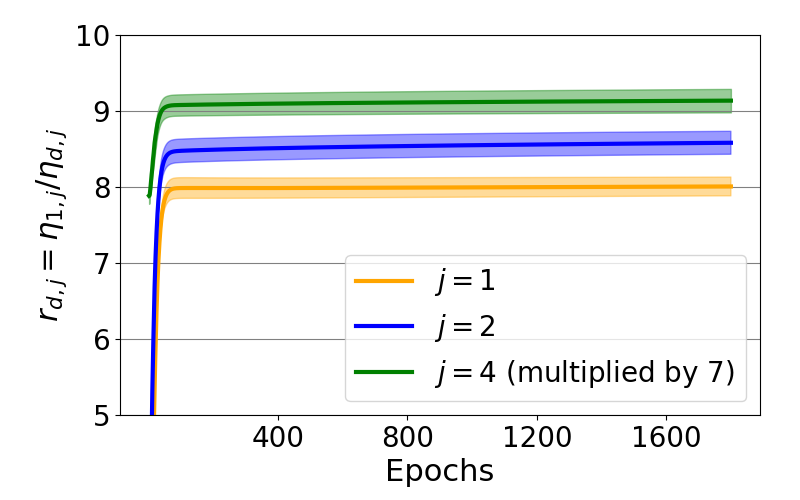}
    \includegraphics[scale=0.3]{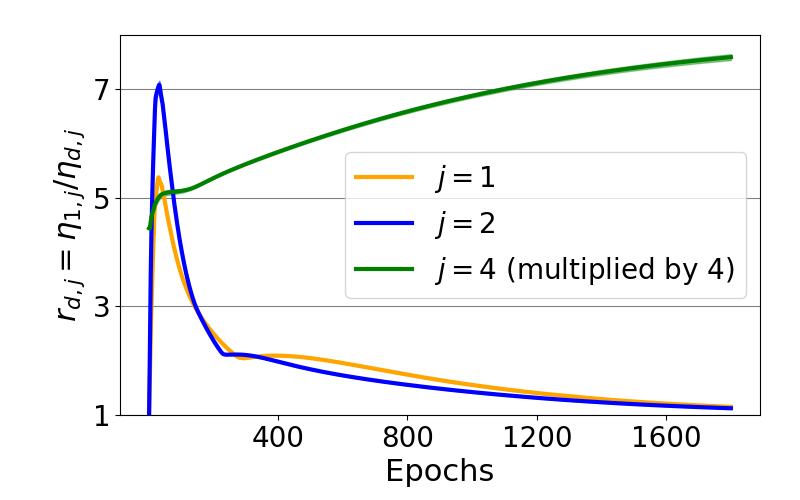}
    \caption{The ratio $r_{d,j}=\eta_{1,j}/\eta_{d,j}$ of singular values (condition number) of weight matrices for layers $j=1,2,4$. (Right) Without regularization (Left) With the regularization based on our bound.}
    \label{fig:singval_mnist}
\end{figure}

\begin{figure}[t]
    \centering
    \includegraphics[scale=0.3]{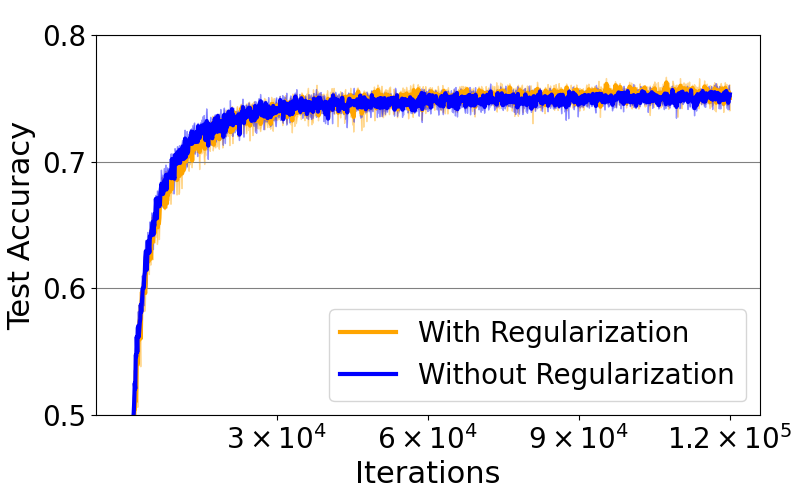}
    \caption{Test accuracy of AlexNet \red{traind} by CIFAR-10 \red{with and without regularization}.}
    \label{fig:acc}
\end{figure}

\begin{figure}[t]
    \centering
    \includegraphics[scale=0.3]{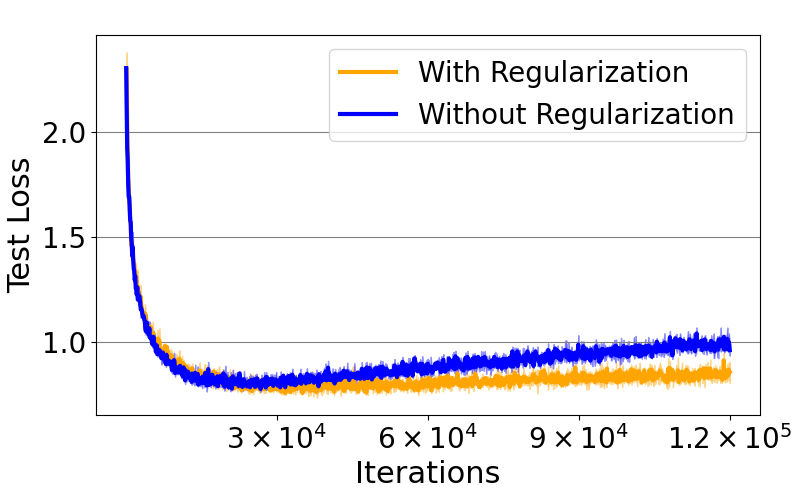}
    \includegraphics[scale=0.3]{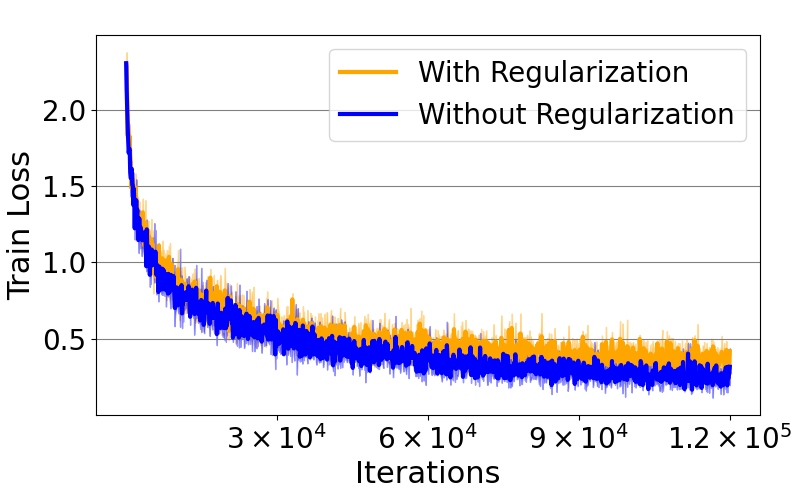}
    \caption{\red{Test and train loss of AlexNet trained by CIFAR-10 with and without regularization. (Right) Test loss (Left) Train loss.}}
    \label{fig:cifar_regularization}
\end{figure}

\subsection{Singular values of the weight matrices}\label{ap:exp_singular}
We constructed an AlexNet~\cite{krizhevsky12} where the ReLU activation function is replaced by the smooth version of Leaky ReLU ($\sigma$ in Section~\ref{app:ex_detail1}) to meet our setting.
For training the network, we used $n=50000$ samples and used the Adam optimizer.
We set the error function as the categorical cross-entropy loss.
We show the test accuracy through the learning process in Figure~\ref{fig:acc}.
In addition to the AlexNet, we also computed the ratio $r_{d,j}$ of the largest and the smallest singular values (condition number) of the weight matrices for the network we used in Section~\ref{app:mnist}.
Since the behavior of the singular values of $W_3$ was unstable and did not have clear patterns, we only show the result for $j=1,2,4$ in Figure~\ref{fig:singval_mnist}.
We scaled the values for $j=4$ for readability.
In the case of the AlexNet, the behavior of singular values of each weight matrix was different depending on the layer.
However, in the case of the network in Section~\ref{app:mnist}, without the regularization, the condition number $r_{d,j}$ stagnates for $j=1,2,4$ through the learning process, and the test accuracy also stagnates.
On the other hand, with the regularization based on our bound, $r_{d,j}$ becomes small for $j=1,2$ as the learning process proceeds by virtue of the regularization.
We can also see that $r_{d,j}$ grows for $j=4$ as the learning process proceeds, which makes the angle $\theta$ in Figure~\ref{fig:angle} large.
As discussed in the last part of Subsection~\ref{subsec:combine}, we can conclude that the regularization transforms the signals in lower layers ($j=1,2$) and makes it easier for them to be extracted in higher layers ($j=4$), and the test accuracy becomes higher than the case without the regularization.
The results in Figures~\ref{fig:mainfig} (c), \ref{fig:singval_mnist}, and \ref{fig:acc} are obtained by three independent runs.

\subsection{Validity of the bound (CIFAR-10)}
\red{We used the same network and the same dataset as Appendix~\ref{ap:exp_singular} and observed the generalization property with and without a regularization based on our result.
We consider the regularization term $\lambda_1\Vert W_j\Vert +\lambda_2/\Vert 0.01I+W_j^*W_j\Vert$, where $\lambda_1=0.1$ and $\lambda_2=0.001$ to make both the largest and the smallest singular values of $W_j$ small (thus, makes our bound small).
Since the AlexNet is composed of convolutional layers, we represented the convolutional layers as matrices.
For the convolution $\sum_{i=1}^n\sum_{j=1}^mf_{k-i,l-j}x_{k,j}$ with a convolutional filter $F=[f_{i,j}]$, we can construct a tensor $\tilde{W}_{i,j,k,l}=f_{k-i,l-j}$.
If $i$ or $j$ is out of the bound of the index of the filter, then we set $f_{i,j}=0$.
We can combine the indices $(i,j)$ and $(k,l)$ in $\tilde{W}_j$ and obtain a matrix $W_j$ that represents the convolution.
Note that since the dimension of $W_j$ is large, setting a regularization term with the determinant of $W_j$ can cause numerical overflows.
Thus, we set $\Vert 0.01I+W_j^*W_j\Vert$ instead of its determinant in the same manner as Appendix~\ref{app:mnist}.
Based on the observation in the last part of Subsection~\ref{subsec:combine}, we set the regularization term for only $j=1,2$.
Figure~\ref{fig:acc} shows the test accuracy obtained with and without the regularization.
The behavior of the accuracy obtained with and without the regularization are similar.
Figure~\ref{fig:cifar_regularization} shows the test and train loss.
We can see that without the regularization, although the train loss becomes small, the test loss becomes large as the iteration proceeds.
On the other hand, with the regularization, the train loss becomes small, and the test loss does not become so large as the case without the regularization.}

\section{Norm of the Sobolev space}\label{ap:sobolev_norm}
Let $p(\omega)=(1+\Vert\omega\Vert^2)^s$ with $s\in\mathbb{N}$.
We can represent the Sobolev norm $\Vert f\Vert_{H_p(\mathbb{R}^d)}$ using the derivatives of $f$ if $s\in\mathbb{N}$.
Indeed, we have
\begin{align*}
\Vert f\Vert_{H_p(\mathbb{R}^d)}^2
&=\int_{\r{d}}\vert\hat{f}(\omega)\vert^2(1+\Vert\omega\Vert^2)^s\mr{d}\omega
=\int_{\r{d}}\vert\hat{f}(\omega)\vert^2\sum_{i=0}^s\binom{s}{i}\Vert\omega\Vert^{2i}\mr{d}\omega\\
&=\int_{\r{d}}\vert\hat{f}(\omega)\vert^2\sum_{i=0}^s\binom{s}{i}\bigg(\sum_{j=1}^d\omega_j^{2}\bigg)^i\mr{d}\omega\\
&=\int_{\r{d}}\vert\hat{f}(\omega)\vert^2\sum_{i=0}^s\binom{s}{i}\sum_{\vert\alpha\vert=i}\binom{i}{\alpha}(\omega^{\alpha})^2\mr{d}\omega
=\sum_{\vert\alpha\vert\le s}\frac{s!}{(s-\vert\alpha\vert)!\alpha!}\int_{\r{d}}\vert\hat{f}(\omega)\omega^{\alpha}\vert^2\mr{d}\omega\\
&=\sum_{\vert\alpha\vert\le s}\frac{s!}{(s-\vert\alpha\vert)!\alpha!}(2\pi)^d\Vert \partial^{\alpha}f\Vert^2_{L^2(\r{d})}.
\end{align*}

\end{document}